\documentclass[11pt]{article}
\usepackage[margin=1in]{geometry}
\usepackage[T1]{fontenc}

\usepackage{natbib}
\usepackage[hyphens]{url}

\usepackage{enumitem}
\usepackage{footnote}
\usepackage[bottom]{footmisc}
\usepackage{xcolor}

\usepackage{algorithm}
\usepackage[noend]{algpseudocode}

\makeatletter
\algnewcommand{\Statexx}[1]{\Statex \hskip\ALG@thistlm #1}
\makeatother
\algnewcommand{\LComment}[1]{\Statexx{\textcolor{gray}{\(\triangleright\) \textit{#1}}}}

\usepackage{rcx}
\usepackage{amssymb}

\hypersetup{
  colorlinks   = true,
  urlcolor     = blue,
  linkcolor    = blue,
  citecolor    = blue
}

\usepackage{mathrsfs}
\usepackage{nicefrac}

\DeclareFontFamily{U}{mathx}{}
\DeclareFontShape{U}{mathx}{m}{n}{<-> mathx10}{}
\DeclareSymbolFont{mathx}{U}{mathx}{m}{n}
\DeclareMathAccent{\widecheck}{0}{mathx}{"71}

\newcommand*\centermathcell[1]{\omit\hfil$\displaystyle#1$\hfil\ignorespaces}

\newcommand{\proj}{\operatorname{\mathrm{proj}}}
\newcommand{\ER}{\operatorname{\mathit{ER}}}

\newcommand{\Prob}{\hyperlink{eq:p}{\mathrm{P}}}
\newcommand{\LProb}{\hyperlink{eq:lp}{\mathrm{LP}}}

\newcommand{\dTV}{D_\textnormal{TV}}
\newcommand{\dKS}{D_\textnormal{KS}}
\newcommand{\BKS}{B_\textnormal{KS}}
\newcommand{\DSP}{\Delta_\textnormal{SP}}

\title{Differentially Private Post-Processing for Fair Regression}
\date{}

\usepackage{authblk}
\ifpdf

\author[1]{Ruicheng Xian}
\author[1]{Qiaobo Li}
\author[2]{Gautam Kamath}
\author[1]{Han Zhao}
\affil[1]{University of Illinois Urbana-Champaign\protect\\{\small\texttt{\{\href{mailto:rxian2@illinois.edu}{rxian2},\href{mailto:qiaobol2@illinois.edu}{qiaobol2},\href{mailto:hanzhao@illinois.edu}{hanzhao}\}@illinois.edu}}\vspace{0.5em}}
\affil[2]{University of Waterloo and Vector Institute\protect\\{\small\texttt{\href{mailto:g@csail.mit.edu}{g@csail.mit.edu}}}}
\else
\author[1]{Ruicheng Xian}
\author[2]{Qiaobo Li}
\author[3]{Gautam Kamath}
\author[4]{Han Zhao}
\affil[1]{University of Illinois Urbana-Champaign\protect\\{\small\texttt{\href{mailto:rxian2@illinois.edu}{rxian2@illinois.edu}}}}
\affil[2]{University of Illinois Urbana-Champaign\protect\\{\small\texttt{\href{mailto:qiaobol2@illinois.edu}{qiaobol2@illinois.edu}}}}
\affil[3]{University of Waterloo and Vector Institute\protect\\{\small\texttt{\href{mailto:g@csail.mit.edu}{g@csail.mit.edu}}}}
\affil[4]{University of Illinois Urbana-Champaign\protect\\{\small\texttt{\href{mailto:hanzhao@illinois.edu}{hanzhao@illinois.edu}}}}
\fi

\begin{document}

\maketitle

\begin{abstract}
This paper describes a differentially private post-processing algorithm for learning fair regressors satisfying statistical parity, addressing privacy concerns of machine learning models trained on sensitive data, as well as fairness concerns of their potential to propagate historical biases. Our algorithm can be applied to post-process any given regressor to improve fairness by remapping its outputs. It consists of three steps: first, the output distributions are estimated privately via histogram density estimation and the Laplace mechanism, then their Wasserstein barycenter is computed, and the optimal transports to the barycenter are used for post-processing to satisfy fairness. We analyze the sample complexity of our algorithm and provide fairness guarantee, revealing a trade-off between the statistical bias and variance induced from the choice of the number of bins in the histogram, in which using less bins always favors fairness at the expense of error.
\end{abstract}

\section{Introduction}

Prediction and forecasting models trained from machine learning algorithms are ubiquitous in real-world applications, whose performance hinges on the availability and quality of training data, often collected from end-users or customers. This reliance on data has raised ethical concerns including fairness and privacy. Models trained on past data may propagate and exacerbate historical biases against disadvantaged demographics, and producing less favorable predictions~\citep{bolukbasi2016ManComputerProgrammer,buolamwini2018GenderShadesIntersectional}, resulting in unfair treatments and outcomes especially in areas such as criminal justice, healthcare, and finance~\citep{barocas2016BigDataDisparate,berk2021FairnessCriminalJustice}. Models also have the risk of leaking highly sensitive private information in the training data collected for these applications~\citep{dwork2014AlgorithmicFoundationsDifferential}.

While there has been significant effort at addressing these concerns, few treats them in combination, i.e., designing algorithms that train fair models in a privacy-preserving manner.  A difficulty is that privacy and fairness may not be compatible: exactly achieving group fairness criterion such as \textit{statistical parity} or \textit{equalized odds} requires precise (estimates of) group-level statistics, but for ensuring privacy, only noisy statistics are allowed under the notion of \textit{differential privacy}.
Resorting to approximate fairness, prior work has proposed private learning algorithms for reducing disparity, but the focus has been on the classification setting~\citep{jagielski2019DifferentiallyPrivateFair,xu2019AchievingDifferentialPrivacy,mozannar2020FairLearningPrivate,tran2021DifferentiallyPrivateFair}.

In this paper, we propose and analyze a differentially private post-processing algorithm for learning \textit{attribute-aware} fair regressors under the squared loss, with respect to the fairness notion of statistical parity. It can take any (privately pre-trained) regressor and remaps its outputs (with minimum deviations) to improve fairness. At a high-level, our algorithm consists of three steps: estimating the output distributions of the regressor from data, computing their Wasserstein barycenter, and the optimal transports~\citep{chzhen2020FairRegressionWasserstein,legouic2020ProjectionFairnessStatistical,xian2023FairOptimalClassification}. To make this process differentially private, we use private histogram density estimates~(HDE) for the distributions via the Laplace mechanism~\citep{diakonikolas2015DifferentiallyPrivateLearning,xu2012DifferentiallyPrivateHistogram}, followed by re-normalization, which introduces additional complexity in our analysis. The choice of the number of bins in the HDE induces a trade-off between the statistical bias and variance for the cost of privacy and fairness. Our theoretical analysis and experiments show that using less bins always improves fairness at the expense of higher error.

\paragraph{Paper Organization.} In~\cref{sec:prelim}, we introduce definitions, notation, and the problem setup. \Cref{sec:alg} describes our private post-processing algorithm for learning fair regressors, with finite sample analysis of the accuracy-privacy-fairness trade-offs in~\cref{sec:thm}. Finally in~\cref{sec:exp}, we empirically explore the trade-offs achieved by our post-processing algorithm on Law School and Communities \& Crime datasets.

\subsection{Related Work}\label{sec:rel}
\paragraph{Differential Privacy.}
Under the notion of differential privacy~\citep{dwork2006CalibratingNoiseSensitivity} and subsequent refinements such as R\'enyi DP~\citep{mironov2017RenyiDifferentialPrivacy}, a variety of private learning algorithms are proposed for settings including regression, classification, and distribution learning. There are private variants of logistic regression~\citep{papernot2017SemisupervisedKnowledgeTransfer}, decision trees~\citep{fletcher2019DecisionTreeClassification}, and linear regression~\citep{wang2018RevisitingDifferentiallyPrivate,covington2021UnbiasedStatisticalEstimation,alabi2022DifferentiallyPrivateSimple}. 
The problem of private distribution learning has been studied for finite-support distributions~\citep{xu2012DifferentiallyPrivateHistogram,diakonikolas2015DifferentiallyPrivateLearning} and parameterized families~\citep{bun2019PrivateHypothesisSelection,aden-ali2021SampleComplexityPrivately}. Lastly, private optimization algorithms including those based on objective perturbation~\citep{chaudhuri2011DifferentiallyPrivateEmpirical,kifer2012PrivateConvexEmpirical} and DP-SGD~\citep{song2013StochasticGradientDescent,bassily2014PrivateEmpiricalRisk,abadi2016DeepLearningDifferential} are proposed for solving convex and general optimization problems.  

\paragraph{Algorithmic Fairness.}
In parallel, the study of algorithmic fairness revolves around the formalization of fairness notions and design of bias mitigation methods~\citep{barocas2023FairnessMachineLearning}. Fairness criteria include those that focus on disparate impact of the model, such as individual fairness~\citep{dwork2012FairnessAwareness}, which asks the model to treat similar inputs similarly, or group-level statistical parity~\citep{calders2009BuildingClassifiersIndependency} and equalized odds~\citep{hardt2016EqualityOpportunitySupervised}; and, those on performance inequality, e.g., accuracy parity~\citep{buolamwini2018GenderShadesIntersectional,chi2021UnderstandingMitigatingAccuracy}, predictive parity~\citep{chouldechova2017FairPredictionDisparate}, and multi-calibration~\citep{hebert-johnson2018MulticalibrationCalibrationComputationallyIdentifiable}.  Fair algorithms are categorized into pre-processing, by removing biased correlation in the training data~\citep{calmon2017OptimizedPreProcessingDiscrimination}; in-processing, that turns the original learning problem into a constrained one~\citep{kamishima2012FairnessAwareClassifierPrejudice,zemel2013LearningFairRepresentations,agarwal2018ReductionsApproachFair,agarwal2019FairRegressionQuantitative}; and post-processing, which remaps the predictions of a trained model post-hoc to meet the fairness criteria~\citep{hardt2016EqualityOpportunitySupervised,pleiss2017FairnessCalibration,chzhen2020FairRegressionWasserstein,zhao2022InherentTradeoffsLearning,xian2023FairOptimalClassification}.

\paragraph{Fairness and Privacy.}
Private fair learning and bias mitigation algorithm are proposed and studied in prior work~\citep{jagielski2019DifferentiallyPrivateFair,xu2019AchievingDifferentialPrivacy,mozannar2020FairLearningPrivate,tran2021DifferentiallyPrivateFair}, but they have so far been focused on the classification setting. It remains an open question on how to achieve private fair regression, and what are the trade-offs between privacy, fairness, and accuracy.

\citet{cummings2019CompatibilityPrivacyFairness} and \citet{agarwal2020TradeOffsFairnessPrivacy} showed that fairness and privacy are incompatible in the sense that no $\varepsilon$-DP algorithm can generally guarantee group fairness on the training set (from which population-level guarantees can be derived via generalization), unless the hypothesis class is restricted to constant predictors.  The argument is that a predictor $f$ that is fair on $S$ may not be fair on its neighbor $S'$, so an $\varepsilon$-DP algorithm that outputs $f$ on $S$ with nonzero probability may also output $f$ on $S'$, which is unfair.  Work on private fair algorithms circumvent this incompatibility by relaxing to high probability guarantees for fairness.  
\citet{bagdasaryan2019DifferentialPrivacyHas} showed that the performance impact of privacy may be more severe on underrepresented groups, resulting in accuracy disparity.

\section{Preliminaries}
\label{sec:prelim}

A regression problem is defined by a joint distribution $\mu$ of the observations $X\in\calX$, sensitive attribute $A\in\calA$ (which has finite support), and the response $Y\in\RR$.  We will use upper case $X,A,Y$ to denote the random variables, and lower case $x,a,y$ instances of them. The goal of this paper is to develop a \textit{privacy-preserving} \textit{post-processing} algorithm for learning (randomized) \textit{fair} regressors. We consider the attribute-aware setting, i.e., the sensitive attribute $A$ is available explicitly during both training and prediction and can be taken as input by the regressor, $f:\calX\times\calA\rightarrow\RR$. 

The risk of a regressor is defined to be mean squared error, and its excess risk is defined with respect to the Bayes regressor, $f^*(x,a)\coloneqq\E[Y\mid X=x,A=a]$:
\begin{align}
    R(f) &:= \E[(f(X,A)-Y)^2], \\
    \ER(f) &:= R(f)-R(f^*) = \E[(f(X,A)-f^*(X,A))^2]
\end{align}
by the orthogonality principle, where $\E$ is taken with respect to $\mu$ (and the randomness of $f$).

Given a (training) dataset $S$ consisting of $n$ i.i.d.\ samples of $(X,A,Y)$, a private learning algorithm minimizes the leakage of any individual's information in its output.  We use the notion of differential privacy~\citep{dwork2006CalibratingNoiseSensitivity}, which limits the influence of any single training sample:
\begin{definition}[Differential Privacy]
    A randomized (learning) algorithm $\mathscr{A}$ is $\varepsilon$-differentially private~(DP) if for all pairs of nonempty neighboring datasets $S,S'$, 
    \begin{equation}
        \Pr(\mathscr{A}(S)\in O) \leq e^\varepsilon \Pr(\mathscr{A}(S')\in O), \quad \forall O\subseteq\operatorname{\mathrm{range}}(\mathscr{A}),
    \end{equation}
    where $\Pr$ is taken with respect to the randomness of $\mathscr{A}$.
\end{definition}

We say two datasets are neighboring if they differ in one entry by substitution (our result also covers insertion and deletion operations, which may have lower sensitivity).  Note that privacy is guaranteed with respect to both $A$ and $X$.

For fairness, we consider statistical parity~\citep{calders2009BuildingClassifiersIndependency}, which requires the output distributions of $f$ conditioned on each group to be similar. The similarity between distributions is measured in Kolmogorov–Smirnov distance, defined for probability measures $p,q$ supported on $\RR$ by
\begin{equation}
    \dKS(p,q) = \sup_{t\in\RR}\envert*{\int_{-\infty}^t\rbr*{p(x)-q(x)}\dif x }.
\end{equation}

\begin{definition}[Statistical Parity]
\label{def:fair}
A (randomized) regressor $f:\calX\times\calA\rightarrow\RR$ satisfies $\alpha$-approximate statistical parity~(SP) if 
\begin{equation}
  \DSP(f)\coloneqq \max_{a,a'\in\calA}\dKS(r_a,r_{a'})\leq\alpha,
\end{equation}
where $r_a$ is the distribution of regressor output $f(X,A)$ conditioned on $A=a$.
\end{definition}

\section{Fair and Private Post-Processing}
\label{sec:alg}

\begin{figure}[t]
  \centering
      \includegraphics[width=\linewidth]{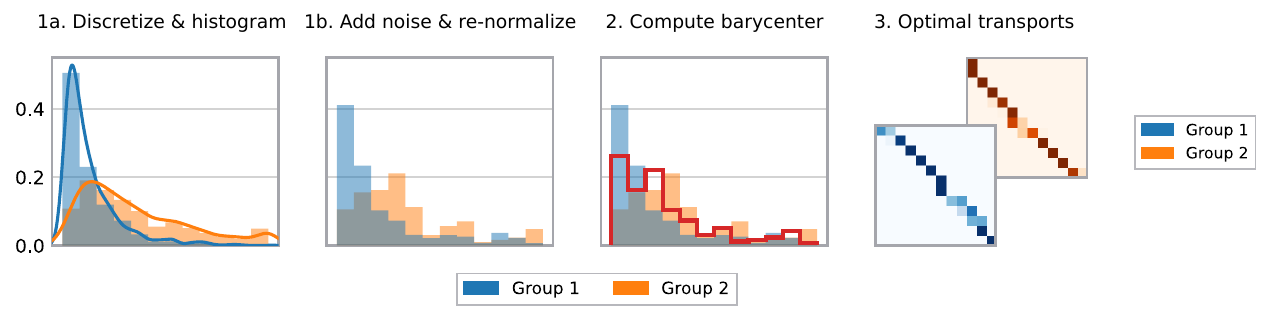}
      \caption{Illustration of the private fair post-processing steps~\ref{it:step.1}--\ref{it:step.3} performed in \cref{alg:post.proc}.  The (randomized) transports to the barycenter are represented by (sparse) $k\times k$ matrices, and the value at the $(i,j)$-th entry is the probability of transporting to bin $j$ given bin $i$.}
      \label{fig:steps}
\end{figure}

We describe a private post-processing algorithm such that, given an unlabeled training dataset $S=\{(x_{i},a_{i})\}_{i\in[n]}$ sampled from $\mu$, and a (privately) pre-trained regressor $f$ (e.g., learned using the algorithms mentioned in \cref{sec:rel}), the algorithm learns a mapping $g$ that transforms the output of $f$ (a.k.a.~post-processing) so that $\bar f\coloneqq g\circ f$ is fair, while preserving differential privacy with respect to $S$.

Compared to in-processing approaches for fairness, post-processing decouples the conflicting goals of fairness and error minimization, since the regressor $f$ can be trained to optimality without any constraint, and then post-processed to satisfy fairness.  We show that this decoupling does not affect the optimality of the resulting fair regressor $\bar f$---the optimal fair regressor can be recovered via post-processing if the algorithm in the pre-training stage learns the Bayes regressor $f^*$ (\cref{thm:opt.fair.equiv}).  Moreover, post-processing has low sample complexity and only requires unlabeled data (\cref{thm:main}), so labeled data (that may be scarce in some domains) can be dedicated entirely to error minimization, which is typically a more difficult problem.

\subsection{Fair Post-Processing with Wasserstein-Barycenters}
\label{sec:post}

Exact statistical parity ($\alpha=0$) requires all groups to have identical output distributions, so finding the fair post-processing mapping $g$ that incurs minimum deviations from the original outputs of $f$ amounts to the following steps~\citep{chzhen2020FairRegressionWasserstein,legouic2020ProjectionFairnessStatistical}:
\begin{enumerate}
    \item\label{it:step.1} Learn the output distributions of $f$ conditioned on each group, $r_a$, for all $a\in\calA$.
    \item\label{it:step.2} Find a common (fair) distribution $q$ that is close to the original distributions (i.e., their barycenter).
    \item\label{it:step.3} Compute the optimal transports $g_a$ from $r_a$ to the barycenter $q$.
\end{enumerate}
The (randomized) transports are then applied to post-process $f$ to obtain a fair attribute-aware regressor $\bar f(x,a)= g_a\circ f(x)$, so that every group has the same output distribution equal to $q$.

Formally, $q$ is called the Wasserstein barycenter of the $r_a$'s, which is a distribution supported on $\RR$ with minimum total distances to the $r_a$'s as measured in Wasserstein distance:
\begin{equation}
  q\in\argmin_{q':\supp(q')\subseteq\RR} \sum_{a\in\calA} w_a W_2^2(r_a, q'), \label{eq:barycenter}
\end{equation}
where $w_a\coloneqq \Pr(A=a)$,
\begin{equation}
  W_2^2(r_a, q) = \min_{\pi_a\in\Pi(r_a, q)}\int_{\RR\times\RR} (y-y')^2\dif\pi_a(y,y'), \label{eq:w2}
\end{equation}
and $\Pi(p,q)=\cbr{\pi:\supp(\pi)\subseteq\RR\times\RR,\int \pi(x,y')\dif y'=p(x), \int\pi(x',y)\dif x'=q(y),\forall x,y }$ is the collection of probability couplings of $p,q$.  The value of $W_2^2(r_a, q)$ naturally represents the squared cost of transforming $r_a$ into $q$, i.e., the minimum amount of output deviations (in squared distance) required to post-process $f$ for group $a$ so that its output distribution is transformed to $q$ (which is a consequence of \cref{lem:fair.equiv}); specifically, this optimal post-processor is given by the optimal transport from $r_a$ to $q$.

\subsection{Generalization to Approximate Statistical Parity}
\label{sec:approx}
To handle approximate statistical parity ($\alpha>0$), in step~\ref{it:step.2}, we first replace the barycenter in \cref{eq:barycenter} by a KS ball of radius $\nicefrac\alpha2$ and relax the problem to finding target output distributions $q_a$ for each group $a$ inside the ball with minimum $W_2^2$ distance to the original distribution $r_a$. Then in step~\ref{it:step.3}, we find the optimal transports $g_a$ from $r_a$ to $q_a$.  So, the problem in \cref{eq:barycenter} is generalized to the approximate SP case with
\begin{flalign}
  \quad\hypertarget{eq:p}{\mathrm{P}}(\{r_a\}_{a\in\calA}, \{w_a\}_{a\in\calA},\alpha): &&
  \argmin_{\substack{q:\supp(q)\subseteq\RR\\ \{q_a\}_{a\in\calA}\subset \BKS(q,\nicefrac\alpha2)}} \sum_{a\in\calA} w_a W^2_2(r_a, q_a), &&&&&&&
\end{flalign}
where $\BKS(q,\nicefrac\alpha2)=\{p:\dKS(p,q)\leq\nicefrac\alpha2\}$ is a KS ball centered at $q$ on the space of probability distributions supported on $\RR$.

When learning from finite samples $S$ (non-privately), we replace $r_a,w_a$ in $\Prob$ with the respective empirical distributions/estimates.  Computationally, solving $\Prob$ on the empirical distributions is as hard as the finite-support Wasserstein barycenter problem $(\alpha=0)$, which can be formulated by a linear program of exponential size in $|\calA|$~\citep{anderes2016DiscreteWassersteinBarycenters,altschuler2021WassersteinBarycentersCan}.
To reduce the complexity, we restrict the support of the barycenter $q$ via discretization; this fixed-support approximation is common in prior work~\citep{cuturi2014FastComputationWasserstein,staib2017ParallelStreamingWasserstein}.

The validity of the post-processing algorithm described above is supported by the fact that, if the regressor being post-processed is Bayes optimal, then the resulting fair regressor is also optimal.  This has been established for the exact SP case in prior work~\citep[e.g.,][Theorem 3]{legouic2020ProjectionFairnessStatistical}, and here we provide a more general result for approximate SP:

\begin{theorem}\label{thm:opt.fair.equiv}
  Let a regression problem be given by a joint distribution $\mu$ of $(X,A,Y)$, denote $w_a=\Pr(A=a)$, the Bayes regressor by $f^*(x,a)=\E[Y\mid X=x,A=a]$, and its output distribution conditioned on group $a$ by $r^*_a$.  Then
  \begin{equation}
    \min_{f:\DSP(f)\leq\alpha}\ER(f) =\hspace{-0.5em} \min_{\substack{q:\supp(q)\subseteq\RR\\ \{q_a\}_{a\in\calA}\subset \BKS(q,\nicefrac\alpha2)}} \sum_{a\in\calA} w_a W^2_2(r^*_a, q_a).
  \end{equation}
\end{theorem}

This shows that the excess risk of the optimal fair regressor is indeed the value of the Wasserstein barycenter problem~($\Prob$) discussed above, and can be achieved by $f^*$ post-processed using the optimal transports to the barycenter.  All proofs are deferred to \cref{app.opt,app.proof}; this result follows from an equivalence between learning regressors and learning post-processings of $f^*$.

\subsection{Privacy via Discretization and Private PMF Estimation}

\begin{algorithm}[t]
    \caption{Fair and Private Post-Processing (Attribute-Aware)}
    \label{alg:post.proc}
    \begin{algorithmic}[1]
            \Require{Regressor $f:\calX\times\calA\rightarrow\RR$, samples $\{x_{i},a_i\}_{i\in[n]}$, interval $[s,t]$, number of bins $k$, fairness tolerance $\alpha$, privacy budget $\varepsilon$}
            \State Let $v_j=\nicefrac{(j-1/2)(t-s)}{k}$ for all $j\in[k]$\Comment{midpoints of the bins}
            \State Define $h(y) = \argmin_{j\in[k]}|y-v_j|$ \Comment{discretizer}
            \State Define $\hat p(a, v_j) = \nicefrac1n \sum_{i\in[n]} \1[h\circ f(x_{i}, a_{i}) = v_j]$ \Comment{empirical joint distribution} \label{ln:dist.1}
            \State Define $\check p(a, v_j) = \hat p(a, v_j) + \mathrm{Laplace}(0,\nicefrac{2}{n\varepsilon})$ \Statex\Comment{Laplace mechanism to get private joint distribution} \label{ln:dist.2}
            \For{$a\in\calA$}
            \State $\tilde w_a \gets \max(\sum_{j = 1}^{k} \check p(a,v_j), 0)$  \Comment{private group marginal distribution} \label{ln:dist.3}
            \State Define $\widecheck{F}_a(v_j) = \nicefrac1{\tilde w_a} \sum_{\ell=1}^j\check p(a, v_\ell)$ \Comment{scaled partial sums} \label{ln:dist.4}
            \State Define $\widetilde{F}_a(v_j) = \begin{cases}
              \proj_{[0,1]}(\nicefrac12(\widecheck F_a(v_{l_j}) + \widecheck F_a(v_{r_j}))) & \text{if } j< k \\
              1 & \text{else}
            \end{cases}$ \label{ln:dist.5}
            \Statexx{where $(l_j,r_j)=\argmax_{l\leq j\leq r}(\widecheck F_a(v_l)-\widecheck F_a(v_r))$}
            \Statex \Comment{$L^\infty$ isotonic regression and clipping to get private CDF}
            \State Define $\tilde p_a(v_j)=\widetilde F_a(v_j)-\widetilde F_a(v_{j-1})$ \Comment{private group conditional PMF} \label{ln:dist.6}
            \EndFor
            \State $(\{\pi_a\}_{a\in\calA}, \tilde q,\{\tilde q_a\}_{a\in\calA}) \gets\LProb(\{\tilde p_a\}_{a\in\calA},\{\tilde w_a\}_{a\in\calA},\alpha)$ \Statex\Comment{compute barycenter and get optimal transports} \label{ln:lp}
            \For{$a\in\calA$} \label{ln.post.2}
            \State Define $g_a(v_j)= \begin{cases}
                v_\ell \textrm{ w.p. } \nicefrac{\pi_a(v_j,v_\ell)}{\tilde p_a(v_j)},\;\;\forall \ell\in[k] &  \textrm{if } \tilde p_a(v_j) > 0  \\
                v_j & \textrm{else}
            \end{cases}$ \Comment{post-processing mappings} \label{ln.post.3}
            \EndFor
            \State \Return $(x,a)\mapsto g_a\circ h\circ f(x,a)$ \Comment{privately post-processed fair regressor} \label{ln.post.4}
    \end{algorithmic}
\end{algorithm}

To make the fair post-processing algorithm in \cref{sec:post} private with respect to $S$, it suffices to perform step~\ref{it:step.1} of estimating the output distributions privately.  This is because the subsequent steps~\ref{it:step.2}~and~\ref{it:step.3} of computing the barycenter (including the target distributions) and the optimal transports only depend on the estimated distributions, so privacy is preserved by the post-processing immunity of DP\@.

To estimate the distributions, one could construct a family of distributions (a.k.a.\ hypotheses) and (privately) choose the most likely one~\citep{bun2019PrivateHypothesisSelection,aden-ali2021SampleComplexityPrivately}, or use non-parametric estimators. For generality, we adopt the latter approach and use (private) histogram density estimator~\citep{diakonikolas2015DifferentiallyPrivateLearning,xu2012DifferentiallyPrivateHistogram}.

Altogether, our fair and private post-processing algorithm is detailed in \cref{alg:post.proc} and illustrated in \cref{fig:steps}.  It takes as inputs the regressor $f$ being post-processed, the samples $S$, the fairness tolerance $\alpha$, and the privacy budget $\varepsilon$. For performing HDE, it also requires specifying an interval\footnote{This is necessary for pure $(\varepsilon,0)$-DP, due to a lower bound by \citet{hardt2010GeometryDifferentialPrivacy} using a packing argument.} (using prior/domain knowledge) that contains the image of the regressor $[s,t]\supseteq f[\supp(\mu_{X,A})]$, and the number of bins $k$.  We describe the algorithm in details below:

\paragraph{Step~\ref{it:step.1} \textnormal{(Estimate Output Distributions)}.}  On line~\ref{ln:dist.1}, we compute the empirical joint distribution of $A$ and $h\circ f(X,A)$, the discretized regressor output (the support is $v\coloneqq(\{\nicefrac{(j-1/2)(t-s)}{k}\}_{j\in[k]})$).  Then, we make this statistics private via the Laplace mechanism on line~\ref{ln:dist.2}, noting that the $L^1$-sensitivity to $\hat p$ is at most $\nicefrac{2}{n}$ (\cref{rem:sensitivity}).\footnote{\label{fn:laplace}The Laplace mechanism could be replaced by, e.g., the Gaussian mechanism, to relax the pure $(\varepsilon,0)$-DP guarantee to approximate $(\varepsilon,\delta)$-DP\@.}

With the private joint distribution, we get a private estimate of the group marginal distribution $\tilde w_a$'s (clipping negative values to zero) on line~\ref{ln:dist.3}, and the private group conditional discretized output distribution $\tilde p_a$'s (with re-normalization) on lines~\ref{ln:dist.4}--\ref{ln:dist.6}.  The re-normalization on $\tilde p_a$ (could also be applied to $\tilde w_a$), required for defining the optimal transport problem in the subsequent step, is done by performing isotonic regression on the partial sums (so that its values are non-decreasing) with clipping to get a valid CDF\@.

We analyze the accuracy of the private estimates:
\begin{theorem}\label{thm:private.estimates}
  Let $p_a(v_j)=\Pr(h\circ f(X,A) = v_j \mid A=a)$ and $w_a=\Pr(A=a)$ for all $a\in\calA$, $j\in[k]$, and let $\tilde p_a$, $\tilde w_a$ denote their privately estimated counterparts in \cref{alg:post.proc}.  Then for all $n\geq \Omega(\max_a\nicefrac1{w_a}\ln\nicefrac1\beta)$, with probability at least $1-\beta$, for all $a\in\calA$,
  \begin{equation}
  \envert{w_a-\tilde w_a} \leq O\rbr*{\frac{\sqrt{k}}{n\varepsilon}\ln\frac{k|\calA|}{\beta}},
  \end{equation}
  and
  \begin{align}
    \|p_a-\tilde p_a\|_1 &\leq  O\rbr*{ \sqrt{\frac{k}{nw_a}\ln\frac{k|\calA|}\beta}  +  \frac{k}{n w_a\varepsilon}\ln\frac{k|\calA|}{\beta} },\\
    \|p_a-\tilde p_a\|_\infty &\leq O\rbr*{ \sqrt{\frac{1}{nw_a}\ln\frac{|\calA|}\beta}   +  \frac{\sqrt{k}}{n w_a\varepsilon}\ln\frac{k|\calA|}{\beta}}, \\
    \dKS(p_a,\tilde p_a) &\leq O\rbr*{ \sqrt{\frac{k}{nw_a}\ln\frac{k|\calA|}\beta}  +  \frac{\sqrt{k}}{n w_a\varepsilon}\ln\frac{k|\calA|}{\beta} }.
  \end{align}
\end{theorem}

The private estimation of PMFs with the Laplace mechanism has been analyzed in prior work~\citep{diakonikolas2015DifferentiallyPrivateLearning,vadhan2017ComplexityDifferentialPrivacy}, but our algorithm performs an extra re-normalization (lines~\ref{ln:dist.4}--\ref{ln:dist.6}) after adding Laplace noise to ensure that the PMF returned is valid.  This makes the analysis of \cref{thm:private.estimates} more involved, because the noise added to each bin can interact during re-normalization.

The rate of $\widetilde O(\sqrt{\nicefrac{k}{n}}+\nicefrac{k}{n\varepsilon})$ for TV distance is consistent with existing results~\citep{diakonikolas2015DifferentiallyPrivateLearning}.  For KS distance, the rate is improved by a $\sqrt k$ factor, which is not surprising as KS is a weaker metric than TV\@.  We note that our analysis in attaining this improvement is made easier with the $L^\infty$ isotonic regression CDF re-normalization scheme, because KS distance is also defined via the CDF\@.

\paragraph{Steps~\ref{it:step.2} and~\ref{it:step.3} \textnormal{(Compute Barycenter and Optimal Transports)}.}  
With the output distributions $\tilde p_a$ estimated privately, we now compute their barycenter and the optimal transports to obtain the fair post-processing mappings $g_a$, i.e., solving $\Prob(\{\tilde p_a\}_{a\in\calA}, \{\tilde w_a\}_{a\in\calA},\alpha)$.  Since the $\tilde p_a$'s are distributions supported on $v$, by restricting the support of the barycenter to $v$,\footnote{\label{fn:restrict}This approximation reduces the size of the barycenter problem to polynomial, at $O(\nicefrac{t-s}{k})$ error, the same as that from discretizing the outputs.} the barycenter and the optimal transports can be computed by solving a linear program with $O(k^2|\calA|)$ variables and constraints (cf.\ $\Prob$ and \cref{eq:w2}):
\begin{flalign}
  \quad\hypertarget{eq:lp}{\mathrm{LP}}(\{p_a\}_{a\in\calA},\{w_a\}_{a\in\calA}, \alpha):
  && \!\argmin_{\substack{\{\pi_a\}_{a\in\calA}\geq0\\q, \{q_a\}_{a\in\calA}\geq0}}\;\;  &\sum_{\mathclap{a\in\calA}}\;\, w_a  \sum_{j,\ell\in[k]} (v_j-v_\ell)^2 \,  \pi_a(v_j,v_\ell) \hspace{-0.8em} \\
    && \centermathcell{\mathrm{s.t.}}\;\;
    & \sum_{\mathclap{\ell\in[k]}} \; \pi_a(v_j,v_\ell) = p_a(v_j), && \forall a\in\calA,\, j\in[k],\\
    &&& \sum_{\mathclap{j\in[k]}} \;\pi_a(v_j,v_\ell) = q_a(v_\ell), && \forall a\in\calA,\, \ell\in[k], \\
    &&& \envert*{\sum_{j\leq \ell} \rbr*{ q_a(v_j) - q(v_j) } }\leq \frac{\alpha}{2}, && \forall a\in\calA,\, \ell\in[k].&&&
\end{flalign}
The constraints enforce that the $\pi_a$'s are couplings, and the target output distributions $\tilde q_a$'s are valid PMFs satisfying the fairness constraint of $\max_{a,a'}\dKS(\tilde q_a,\tilde q_{a'})\leq\alpha$.

\paragraph{Post-Processed Fair Regressor.}
On \cref{ln.post.2,ln.post.3,ln.post.4}, the (randomized) optimal transports $g_a$ from $\tilde p_a$ to $\tilde q_a$ are extracted by reading off from the optimal couplings $\pi_a$ (represented by $k\times k$ matrices), and are used to construct the fair regressor, $\bar f(x,a)= g_a\circ h\circ f(x,a)$.

Given an input $(x,a)$, the fair regressor $\bar f$ obtained from \cref{alg:post.proc} calls $f$ to make a prediction $y=f(x,a)$, discretizes it to get $\tilde y=h(y)$, and then uses the optimal transport $g_a$ of the respective group to post-process and return a fair prediction, $\bar y=g_a(\tilde y)$.  

\subsection{Statistical Analysis}
\label{sec:thm}
\Cref{alg:post.proc} is $\varepsilon$-DP by the post-processing theorem of DP~\citep[Proposition 2.1]{dwork2014AlgorithmicFoundationsDifferential}, because its output depends on $S$ only via the private statistics $\tilde p$ on line~\ref{ln:dist.2} that is $\varepsilon$-DP\@.\footref{fn:laplace}

We analyze the suboptimality (or \textit{fair excess risk}) of the post-processed fair regressor and provide fairness guarantee:

\begin{theorem}\label{thm:main}
  Let regressor $f$ be given along with samples $S=\{(x_i,a_i)\}_{i\in[n]}$ of $\mu_{X,A}$. Denote $L= t-s$, assume $L\leq1$, and $Y,f(X,A)\in[s,t]$ almost surely. Let $\bar f$ denote the fair regressor returned from \cref{alg:post.proc} on $(f,S,[s,t],k,\alpha,\varepsilon)$. Then for all $n\geq \Omega(\max_a\nicefrac1{w_a}\ln\nicefrac{|\calA|}\beta)$, with probability at least $1-\beta$,
    \begin{align}
  R(\bar f) - R(\bar f^*) 
    \leq O\rbr*{ \sqrt{\frac{k|\calA|}{n}\ln\frac{k|\calA|}\beta}  +  \frac{k|\calA|}{n\varepsilon}\ln\frac{k|\calA|}{\beta} }  + \frac{8}{k} + 5\E\sbr*{\envert*{f(X,A) - f^*(X,A)}}
\end{align}
  where $\bar f^*=\argmin_{f':\DSP(f')\leq\alpha}R(f')$ is the optimal fair regressor, $f^*$ is the (unconstrained) Bayes regressor, and
  \begin{align}
    &\DSP(\bar f) \leq \alpha +\max_{a\in\calA} O\rbr*{ \frac{\sqrt{k}}{n w_a\varepsilon}\ln\frac{k|\calA|}{\beta} +  \sqrt{\frac{k}{nw_a}\ln\frac{k|\calA|}\beta}  }.
  \end{align}
\end{theorem}

The risk bound reflects four potential sources of error: the first term is finite sample estimation error, the second term is due to the noise added for $\varepsilon$-DP, the third term is the error introduced by discretization, and the last term is the $L^1$ excess risk of the regressor $f$ being post-processed (carrying over the error of the pre-trained regressor). 

The first three terms of the risk bound (and last two in the fairness bound) are attributed to the accuracy of the private distribution estimate using HDE (cf.~\cref{thm:private.estimates}). In particular, a trade-off between the statistical bias and variance is incurred by the choice of $k$, the number of bins:
\begin{equation}
  \widetilde O\Bigg( \underbrace{\sqrt{\frac{k}{n}} + \frac{k}{n\varepsilon}}_{\textrm{variance}} + \underbrace{\frac{1}{k}}_{\textrm{bias}} \Bigg).
\end{equation}
Using too few bins leads to a large discretization error (statistical bias), whereas using too many suffers from large variance due to data sampling and the noise added for privacy.

The cost of privacy is dominated by the estimation error as long as $n\gtrsim \max_a\nicefrac{k}{w_a\varepsilon^2}\geq \nicefrac{k|\calA|}{\varepsilon^2}$.  In which case, the choice of $k=\widetilde\Theta(n^{1/3})$ is optimal for MSE, which is consistent with classical non-parametric estimation results of HDE~\citep{rudemo1982EmpiricalChoiceHistograms}.

The fairness bound, on the other hand, only contains variance terms but not the statistical bias, therefore using fewer bins is always more favorable for fairness at the expense of MSE (e.g., in the extreme case of $k=1$, the post-processor outputs a constant value, which is exactly fair).  This suggests that when $n$ is small,  $k$ can be decreased to reduce variance for higher levels of fairness.

The optimal choice of $k$ (in combination with $\alpha$) is data-dependent, and is typically tuned on a validation split, however; extra care should be taken as this practice could, in principle, violate differential privacy. In our experiments, we sweep $\alpha$ and $k$ to empirically explore the trade-offs between error, privacy and fairness attainable with our \cref{alg:post.proc}.  
This is common practice in differential privacy research~\citep{mohapatra2022RoleAdaptiveOptimizers}.
Regarding the selection of hyperparameters while preserving privacy, we refer readers to~\citet{liu2019PrivateSelectionPrivate,papernot2022HyperparameterTuningRenyi}.

\section{Experiments}
\label{sec:exp}

In this section, we evaluate the private and fair post-processing algorithm described in \cref{alg:post.proc}.

We do not compare to other algorithms because we are not aware of any existing private algorithms for learning fair regressors.  Although it may be possible to adopt some existing algorithms to this setting, e.g., using DP-SGD as in~\citep{tran2021DifferentiallyPrivateFair}, making them practical and competitive requires care, and is hence left to future work.  Our post-processing algorithm is based on~\citep{chzhen2020FairRegressionWasserstein}; their paper is referred to for empirical comparisons to in-processing algorithms under the non-private setting ($\varepsilon=\infty$).

\paragraph{Setup.}
Since the excess risk can be decomposed according to the pre-training and post-processing stages (\cref{thm:main}, where $\E[|f-f^*|]$ is carried over from the suboptimality of the regressor $f$ trained in the pre-training stage),
we will simplify our experiment setup to isolate and focus on the performance and trade-offs induced by post-processing alone.  This means we will access the ground-truth responses and directly apply \cref{alg:post.proc} (i.e., $X=Y$ and $f=\mathrm{Id}$, whereby $\E[|f-f^*|]=0$).\footnote{\label{fn:code}Code is available at \url{https://github.com/rxian/fair-regression}.}

\paragraph{Datasets.}
The datasets are randomly split 70-30 for training (i.e., post-processing) and testing.

\textit{Communities \& Crime~\textnormal{\citep{redmond2002DatadrivenSoftwareTool}}.}~~It contains the socioeconomic and crime data of communities in US, and the task is to predict the rate of violent crimes per 100k population ($Y\in[0,1]$).  The sensitive attribute is an indicator for whether the community has a significant presence of a minority population ($|\calA|=2$).  The total size of the dataset is 1,994, and the number of training examples for the smallest group is 679 ($\approx nw_a$).

\textit{Law School~\textnormal{\citep{wightman1998LSACNationalLongitudinal}}.}~~This dataset contains the academic performance of law school applicants, and the task is to predict the student's undergraduate GPA ($Y\in[1,4]$).  The sensitive attribute is race ($|\calA|=4$), and the dataset has size 21,983.  The smallest group has 628 training examples.

\subsection{Results}\label{sec:res}

\begin{figure}[t]
\centering
    \includegraphics[width=0.97\linewidth]{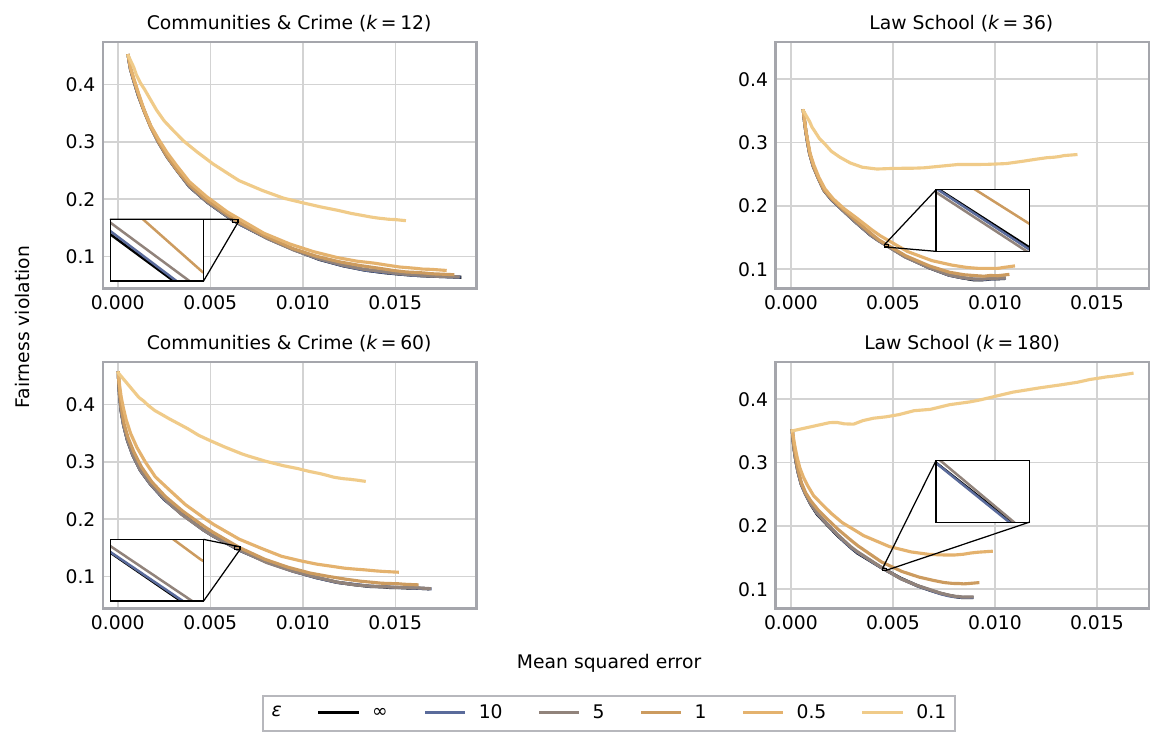}
    \caption{Error-privacy-fairness trade-offs achieved by \cref{alg:post.proc} by sweeping $\alpha$ under the indicated number of bins $k$, with different privacy budgets $\varepsilon$. Fairness violation is measured in KS distance as defined in \cref{def:fair} ($\DSP$). Average of 50 random seeds (also for \cref{fig:results.sweep,fig:results.tradeoff}).}
    \label{fig:results.approximate}
\end{figure}

In \cref{fig:results.approximate}, we show the MSE and fairness trade-offs (in KS distance; $\DSP$, \cref{def:fair}) attained with \cref{alg:post.proc} under various settings of $\alpha$.  The main observations are:

\begin{enumerate}
    \item The cost of discretization (indicated by the horizontal distance from $0$ to the top-left starting points of the curves) can be expected to be insignificant compared to fairness, unless the model is already very fair without post-processing.

    \item Although the amount of data available for post-processing is small by modern standards, the results are very insensitive to the privacy budget until the highest levels of DP are demanded or  very large $k$ is used.
    
    This is because according to \cref{thm:main}, the error attributed to DP noise is dominated by estimation error when $n\gtrsim \max_a\nicefrac{k}{w_a\varepsilon^2}$.  Note that in our experiments, we have $n\gg\nicefrac{k|\calA|}{\varepsilon^2}$ on both datasets except for $\varepsilon=0.1$, and $\varepsilon=0.5$ when $k=180$.  Using more bins increases the cost of privacy due to variance, as we will show in \cref{fig:results.sweep}.  
    
    \item The right end of each line is obtained with setting $\alpha=0$.  Relaxing it to larger values could give better trade-offs, especially when $n,\varepsilon$ are small, because in these cases the estimated distributions can be inaccurate due to estimation error and the noise added, so aiming for exact SP may fit to data artifacts or the noise rather than the true signal, causing MSE to increase without actual improvements to fairness.

\end{enumerate}

\begin{figure}[t]
    \centering
    \begin{minipage}[t]{0.48\linewidth}
        \centering
        \includegraphics[width=1\linewidth]{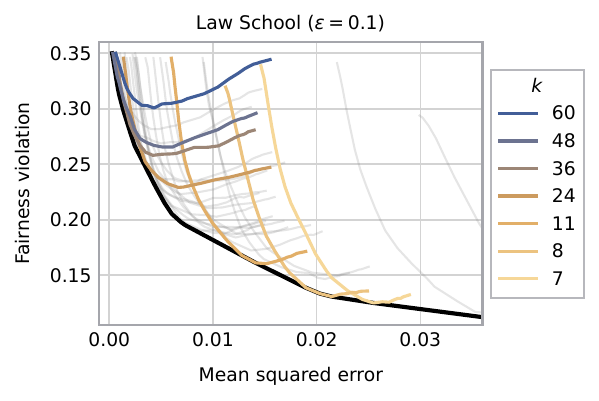}
        \caption{Error-fairness trade-offs achieved by \cref{alg:post.proc} on the Law School dataset by sweeping $\alpha$ and $k$ for $\varepsilon=0.1$. The black line is the lower envelope, and ends on the right at $(0.6772,0)$ (outside the cropped figure).}
        \label{fig:results.sweep}
    \end{minipage}%
    \hfill
    \begin{minipage}[t]{0.48\linewidth}
        \centering
        \includegraphics[width=1\linewidth]{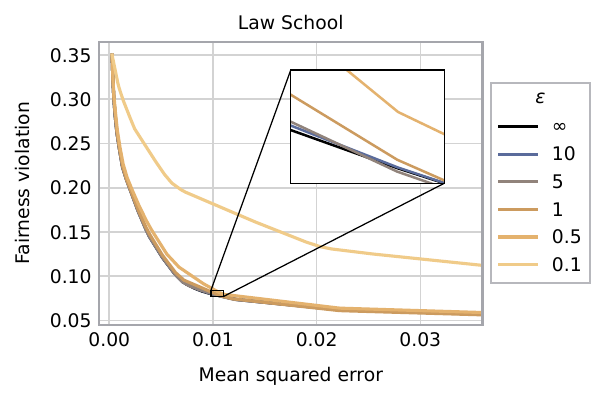}
        \caption{Error-privacy-fairness trade-offs achieved by \cref{alg:post.proc} on the Law School dataset by sweeping $\alpha$ and $k$ and taking the lower envelope.  The line for $\varepsilon=0.1$ is the black line in \cref{fig:results.sweep}.  All lines meet and end on the right at $(0.6772,0)$.}
        \label{fig:results.tradeoff}
\end{minipage}%
\end{figure}

\paragraph{Trade-off Between Statistical Bias and Variance.}

Recall from the analysis and discussion in \cref{sec:thm} that the MSE of \cref{alg:post.proc} exhibits a trade-off between the statistical bias and variance from the choice of $k$: using more bins lowers discretization error but suffers from larger estimation error and more noise added for DP, and vice versa.  On the other hand, the fairness only depends on the variance.  Hence, one can decrease $k$ to achieve higher levels of fairness at the expense of MSE\@.

In \cref{fig:results.sweep}, we plot the trade-offs on the Law School dataset for $\varepsilon=0.1$ under different settings of $k$.  On the right extreme, exact SP ($\DSP=0$) is achieved by setting $k=1$, but with a high $0.6772$ MSE\@.  As expected:
\begin{enumerate}
    \item Smaller settings of $k$ lead to smaller $\DSP$, but it comes with higher discretization error, as reflected in the rightward-shifting starting points (i.e., results with $\alpha=\infty$, i.e., not post-processed for fairness).  
    \item Using more bins may result in worse trade-offs compared to less bins due to large variance. Note that the trade-offs with $k=60$ bins are almost completely dominated by those with $k=36$ bins.
    \item The general trend is that, to achieve the best trade-offs (i.e., the lower envelope), use smaller $(\alpha,k)$ when aiming for higher levels of fairness, and larger $(\alpha,k)$ for smaller MSE (but less fairness).

\end{enumerate}

\paragraph{Error-Privacy-Fairness Trade-Off.}

The black line in \cref{fig:results.sweep} shows the optimal error-fairness trade-offs attainable with \cref{alg:post.proc} for $\varepsilon=0.1$ from sweeping $\alpha$ and $k$ and taking the lower envelope (segments of the line not reached by any $k$ can be achieved by combining two regressors via randomization, although obtaining them via post-processing requires double the privacy budget).  

We repeat this experiment for all settings of $\varepsilon$ on the Law School dataset, and plot their lower envelopes in \cref{fig:results.tradeoff}.  This illustrates the Pareto front of the error-privacy-fairness trade-offs achieved by \cref{alg:post.proc}.  Demanding stricter privacy degrades the trade-offs between error and fairness.

Lastly, we remark that while \cref{fig:results.approximate,fig:results.sweep,fig:results.tradeoff} illustrate the trade-offs that can be possibly attained with our post-processing algorithm from varying the hyperparameters $\alpha$ and $k$, selecting the desired trade-off requires tuning them (privately) on a validation set.  Readers are referred to~\citep{liu2019PrivateSelectionPrivate,papernot2022HyperparameterTuningRenyi} on the topic of differentially private hyperparameter tuning.

\section{Conclusion and Future Work}
In this paper, we described and analyzed a private post-processing algorithm for learning attribute-aware fair regressors, in which privacy is achieved by performing histogram density estimates of the distributions privately, and fairness by computing the optimal transports to their Wasserstein barycenter.  We evaluated the error-privacy-fairness trade-offs attained by our algorithm on two datasets.

Although we only studied the attribute-aware setting, that is, we have explicit access to $A$ during training and when making predictions, our post-processing algorithm could be extended to the attribute-blind setting, where $A$ is only available in the training data.  This requires training an extra predictor for the sensitive attribute, $\widehat P_A = \widehat \Pr(A\mid X)\in\Delta^{|\calA|-1}$, to estimate the joint distribution of $(\widehat Y,\widehat P_A)$ (vs.\ the joint of $(\widehat Y,A)$ estimated in \cref{alg:post.proc} for the attribute-aware setting), and modifying $\Prob$ (and $\LProb$) to use predicted group membership $\widehat P_A$ rather than the true $A$.  This has a higher sample complexity, and the fairness of the post-processed regressor will additionally depend on the accuracy of $\widehat P_A$.  We leave the implementation and analysis of this extension to future work.

\section*{Acknowledgements}

GK is supported by a Canada CIFAR AI Grant, an NSERC Discovery Grant, and an unrestricted gift from Google. HZ is partially supported by a research grant from the Amazon-Illinois Center on AI for Interactive Conversational Experiences (AICE) and a Google Research Scholar Award.

\section*{Broader Impacts}

This paper continues the study of privacy and fairness in machine learning, and fills the gap in prior work on private and fair regression.  Since the setting is well-established, and we have theoretically analyzed the risk of our algorithm, we do not find outstanding societal consequences for discussion.

\bibliographystyle{plainnat-eprint}
\bibliography{references}

\newpage
\appendix

\section{Excess Risk of the Optimal Fair Regressor}\label{app.opt}

This section proves \cref{thm:opt.fair.equiv}, that the excess risk of the optimal attribute-aware fair regressor can be expressed as the sum of Wasserstein distances from the output distributions $r^*_a$ of the Bayes regressor $f^*$ conditioned on each group $a\in\calA$ to their barycenter $q$.  We will cover the approximate fairness setting using the same analysis in~\citep{xian2023FairOptimalClassification}.

The result is a direct consequence of \cref{lem:fair.equiv}---but before stating which, because the fair regressor is randomized, to make the discussions involving randomized functions rigorous, we provide a formal definition for them with the Markov kernel.

\begin{definition}[Markov Kernel]\label{def:markov.kernel}
  A Markov kernel from a measurable space $(\calX,\calS)$ to $(\calY,\calT)$ is a mapping $\calK:\calX\times \calT\rightarrow[0,1]$, such that $\calK(\cdot,T)$ is $\calS$-measurable, $\forall T\in\calT$, and $\calK(x,\cdot)$ is a probability measure on $(\calY,\calT)$, $\forall x\in\calX$.
\end{definition}

\begin{definition}[Randomized Function]\label{def:rand.fn}
  A randomized function $f:(\calX,\calS)\rightarrow(\calY,\calT)$ is associated with a Markov kernel $\calK:\calX\times \calT \rightarrow[0,1]$, such that $\forall x\in\calX, T\in\calT$, $\Pr(f(x)\in T) = \calK(x,T)$.
\end{definition}

\begin{definition}[Push-Forward Distribution]\label{def:push.forward}
    Let $p$ be a measure on $(\calX,\calS)$ and $f:(\calX,\calS)\rightarrow(\calY,\calT)$ a randomized function with Markov kernel $\calK$. The push-forward of $p$ under $f$, denoted by $f\sharp p$, is a measure on $\calY$ given by $f\sharp  p(T) =  \int_{\calX} \calK(x,T)  \dif p(x)$, $\forall T\in\calT$.
\end{definition}

Now, we state the lemma of which \cref{thm:opt.fair.equiv} is a direct consequence; it says that given any (randomized) regressor $f$ with a particular shape $q$, one can derive a regressor $g\circ f^*$ from the Bayes regressor $f^*$ that has the same shape and excess risk ($g$ is a randomized function with Markov kernel $\calK(y,T)=\nicefrac{\pi(y,T)}{\pi(y,\RR)}$ where $\pi$ is given in \cref{eq:coupling.equiv}):

\begin{lemma}\label{lem:fair.equiv}
  Let a regression problem be given by a joint distribution $\mu$ of $(X,Y)$, denote the Bayes regressor by $f^*:\calX\rightarrow\RR$ and $r^*= f^*\sharp\mu_X$, and let $q$ be an arbitrary distribution on $\RR$.  Then, for any randomized regressor $f$ with Markov kernel $\calK$ satisfying $f\sharp\mu_X=q$,
  \begin{equation}
    \pi(y^*,y) = \int_{{f^*}^{-1}(y^*)} \calK(x,y)\dif\mu_X(x)\label{eq:coupling.equiv}
  \end{equation}
  (where ${f^*}^{-1}(y^*)\coloneqq\{x\in\calX:f^*(x)=y^*\}$) is a coupling $\pi\in\Pi(r^*,q)$ that satisfies
  \begin{equation}
    \ER(f) = \int (y^*-y)^2 \dif \pi(y^*,y). \label{eq:fair.equiv}
  \end{equation}
  Conversely, for any $\pi\in\Pi(r^*,q)$, the randomized regressor $f$ with Markov kernel
  \begin{equation}
    \calK(x,T) = \frac{\pi(f^*(x),T)}{\pi(f^*(x),\RR)}
  \end{equation}
  satisfies $f\sharp\mu_X=q$ and \cref{eq:fair.equiv}.
\end{lemma}

\begin{proof}
  For the first direction, note that
  \begin{align}
    \ER(f) 
    &= \E\sbr{\rbr{f^*(X)-f(X)}^2} \\
    &= \int_{\RR\times\RR} (y^*-y)^2 \Pr( f^*(X) = y^*, f(X)=y )\dif(y^*,y) \\
    &= \int_{\RR\times\RR} (y^*-y)^2 \rbr*{\int_\calX \Pr( f^*(X) = y^*, f(X)=y, X=x )\dif x}\dif(y^*,y) \\
    &= \int_{\RR\times\RR} (y^*-y)^2 \rbr*{\int_{{f^*}^{-1}(y^*)} \Pr( f(X)=y, X=x )\dif x}\dif(y^*,y) \\
    &= \int_{\RR\times\RR} (y^*-y)^2 \rbr*{\int_{{f^*}^{-1}(y^*)} \Pr( f(X)=y\mid X=x )\dif \mu_X(x)}\dif(y^*,y) \\
    &= \int_{\RR\times\RR} (y^*-y)^2 \pi(y^*,y) \dif(y^*,y) \label{eq:fair.equiv.er}
  \end{align}
  as desired, where line 4 is because $\Pr( f^*(X) = y^*, f(X)=y, X=x ) = \1[f^*(x)=y^*]\Pr( f(X)=y, X=x )$ as $f^*$ is deterministic.  We also verify that the constructed $\pi$ is a coupling:
  \begin{align}
    \int_\RR\pi(y^*,y)\dif y 
    &= \int_\RR \int_{{f^*}^{-1}(y^*)} \calK(x,y)\dif\mu_X(x)\dif y \\
    &= \int_{{f^*}^{-1}(y^*)} \int_\RR \calK(x,y) \dif y \dif\mu_X(x) \\
    &= \int_{{f^*}^{-1}(y^*)} \dif\mu_X(x)\\
    &= \Pr(f^*(X)=y^*) \\
    &= r^*(y^*)
  \end{align}
  by \cref{def:markov.kernel,def:push.forward}, and
  \begin{align}
    \int_\RR\pi(y^*,y)\dif y^*
    &= \int_\RR \int_{{f^*}^{-1}(y^*)} \calK(x,y)\dif\mu_X(x)\dif y^* \\
    &= \int_\calX \calK(x,y) \dif\mu_X(x) \\
    &= \int_\calX \Pr(f(X)=y\mid X=x) \dif\mu_X(x)\\
    &= \Pr(f(X)=y) \\
    &= q(y)
  \end{align}
  by \cref{def:rand.fn} and the assumption that $f\sharp\mu_X=q$.

For the converse direction, it suffices to show that the Markov kernel constructed for $f$ satisfies the equality in \cref{eq:coupling.equiv}, which would immediately imply $f\sharp\mu_X=q$, and \cref{eq:fair.equiv} with the same arguments in \cref{eq:fair.equiv.er}.  Let $y,y\in\RR$ and $z\in {f^*}^{-1}(y^*)$ be arbitrary, then
\begin{align}
  \pi(y^*,y) 
  &= \frac{\pi(y^*,y)}{\pi(y^*,\RR)}\pi(y^*,\RR) \\
  &= \frac{\pi(f^*(z),y)}{\pi(f^*(z),\RR)}\pi(y^*,\RR) \\
  &=\calK(z,y)\pi(y^*,\RR) \\
  &=\calK(z,y)r^*(y^*) \\
  &=\calK(z,y)\int_{{f^*}^{-1}(y^*)}\dif\mu_X(x) \\
  &=\int_{{f^*}^{-1}(y^*)} \calK(z,y) \dif\mu_X(x) \\
  &=\int_{{f^*}^{-1}(y^*)} \calK(x,y) \dif\mu_X(x),
\end{align}
where line 3 is by construction of $\calK$, line 4 by the assumption that $\pi\in\Pi(r^*,q)$, and the last line is because $\calK(x,y)=\calK(z,y)$ for all $x\in{f^*}^{-1}(y^*)$, also by construction.
\end{proof}

This lemma allows us to formulate the problem of finding the optimal regressor under a shape constraint $q$ as that of finding the optimal coupling $\pi\in\Pi(r^*,q)$ with the squared cost (and $\pi$ can be used to derive the regressor $g\circ f^*$ that achieves the minimum of the original problem).  Because statistical parity is a shape constraint on the regressors, we can leverage this lemma to prove \cref{thm:opt.fair.equiv}:

\begin{proof}[Proof of \cref{thm:opt.fair.equiv}]
  Because we are finding an attribute-aware fair regressor, $f:\calX\times\calA\rightarrow\RR$, we can optimize the components corresponding to each group independently, i.e., $f_a\coloneqq f(\cdot,a)$, $\forall a\in\calA$. Denote the excess risk conditioned on group $a$ by $\ER_a(f_a)=\E\sbr{(f_a(X)-Y)^2\mid A=a}$, and the marginal distribution of $X$ conditioned on $A=a$ by $\mu_{X\mid a}$, then
    \begin{align}
    \min_{f:\DSP(f)\leq\alpha}\ER(f)  
    &=\min_{f:\dKS(f_a\sharp\mu_{X\mid a},f_{a'}\sharp\mu_{X\mid a})\leq \alpha}\ER(f) \\
    &=  \min_{f:\dKS(r_{a}(f),r_{a'}(f))\leq \alpha}\sum_{a\in\calA} w_a \ER_a(f_a) \\
    &= \min_{\{f_a\}_{a\in\calA}:\dKS(f_a\sharp \mu_{X\mid a},f_{a'}\sharp \mu_{X\mid a'})\leq \alpha}\sum_{a\in\calA} w_a \ER_a(f_a) \\
    &= \min_{\{q_a\}_{a\in\calA}:\dKS(q_a,q_{a'})\leq \alpha}\sum_{a\in\calA} w_a \min_{f_a:f_a\sharp \mu_{X\mid a}=q_a} \ER_a(f_a) \\
    &= \min_{\substack{q:\supp(q)\subseteq\RR\\ \{q_a\}_{a\in\calA}\subset \BKS(q,\nicefrac\alpha2)}}\sum_{a\in\calA}w_a \min_{f_a:f_a\sharp \mu_{X\mid a}=q_a} \ER_a(f_a)
  \end{align}
  (the proof that $\dKS(\mu,\nu)\leq \alpha \iff \exists q \text{ s.t.\ } \mu,\nu\in \BKS(q,\nicefrac\alpha2)$ is omitted; a hint for the forward direction is that such a $q$ can be constructed by averaging the CDFs of $\mu,\nu$), where, by \cref{lem:fair.equiv} and the definition of Wasserstein distance,
  \begin{align}
    \min_{f_a:f_a\sharp \mu_{X\mid a}=q_a} \ER_a(f_a) = \min_{\pi_a\in\Pi(r^*_a, q_a)} \int (y^*-y)^2 \dif \pi_a(y^*,y) = W_2^2(r^*_a, q_a),
  \end{align}
  and the theorem follows by plugging this back into the previous equation.
\end{proof}

\section{Proofs for Section~\ref{sec:alg}}
\label{app.proof}

We analyze the $L^1$ sensitivity for our notion of neighboring datasets described in \cref{sec:prelim}, and provide the proofs to \cref{thm:private.estimates,thm:main}, in that order.

\begin{remark}\label{rem:sensitivity}
For nonempty neighboring datasets $S,S'$ that differ in at most one entry by insertion, deletion or substitution, the $L^1$ sensitivity to the empirical PMF is at most $\nicefrac{2}{n}$.  

Let $\hat p, \hat p'$ denote the empirical PMFs of $S$ and $S'$, respectively, assume w.l.o.g.\ that they have two coordinate, and the insertion/deletion takes place in the first coordinate.  Denote $n= |S|$, $n_1= n\hat p_1$, and $n_2= n\hat p_2$.
\begin{itemize}
    \item (Insertion).~~The sensitivity is
\begin{align}
\|\hat p-\hat p'\|_1&
=\envert*{\frac{n_1}{n}-\frac{n_1+1}{n+1}} + \envert*{\frac{n_2}{n} - \frac{n_2}{n+1}}\\
&=\envert*{\frac{n_1(n+1)-(n_1+1)n}{n(n+1)}}+\envert*{\frac{n_2(n+1)-n_2n}{n(n+1)}}\\
&=\envert*{\frac{n_1-n}{n(n+1)}}+\envert*{\frac{n_2}{n(n+1)}}\\
&=2\envert*{\frac{n_2}{n(n+1)}}\\
&\leq \frac{2}{n}.
    \end{align}
    \item (Deletion).~~Similarly, $\|\hat p-\hat p'\|_1=|\nicefrac{n_1}{n}-\nicefrac{n_1-1}{n-1}| + |\nicefrac{n_2}{n} - \nicefrac{n_2}{n-1}|=2|\nicefrac{n_2}{n(n-1)}|\leq \nicefrac{2}{n}$, because $n_2\leq n-1$.
    \item (Substitution).~~$\|\hat p-\hat p'\|_1=|\nicefrac{n_1}{n}-\nicefrac{n_1-1}{n}| + |\nicefrac{n_2}{n} - \nicefrac{n_2+1}{n}|=|\nicefrac{1}{n}| + |\nicefrac{1}{n}|=\nicefrac{2}{n}$.
\end{itemize}
\end{remark}

For the proofs of the theorems, several technical results are required.  First, a concentration bound of i.i.d.\ sum of Laplace random variables based on the following fact~\citep{li2023TailBoundsSums}, which is due to $\sqrt{2Y_j}Z\sim\mathrm{Laplace}(0,1)$ and $\sum_{j=1}^k a_j Z_j\sim \calN(0,\sum_{j=1}^k a^2_j)$ for $Z_1,\dots,Z_k\sim\calN(0,1)$ and $a_1,\dots,a_k\geq0$:

\begin{proposition}\label{prop:laplace.dist}
  Let $X_1,\dots,X_k\sim\mathrm{Laplace}(0,1)$, $Y_1,\dots,Y_k\sim\mathrm{Exponential}(1)$, and $Z\sim\calN(0,1)$ all be independent, then $\sum_{j=1}^k X_j$ has the same distribution as $\rbr{2\sum_{j=1}^k  Y_j}^{1/2}Z$.
\end{proposition}

\begin{lemma}\label{cor:laplace}
  Let independent $X_1,\dots,X_k\sim\mathrm{Laplace}(0,1)$, then for all $t\geq0$, with probability at least $1-\beta$, $\envert{\sum_{j=1}^kX_j}\leq 2\sqrt{k}\ln\nicefrac{2k}\beta$.
\end{lemma}

\begin{proof}
  Using \cref{prop:laplace.dist}, we bound $\sum_{j=1}^kX_j$ by analyzing $\sqrt{\sum_{j=1}^k Y_j}\envert{ Z}$.
  For all $t\geq0$, 
  \begin{align}
    \Pr\rbr*{\sum_{j=1}^k Y_j \geq t} &\leq \Pr\rbr*{\exists j \text{ s.t.\ } Y_j \geq \frac tk} 
    \leq k \Pr\rbr*{Y_1 \geq \frac tk}
    \leq k \exp\rbr*{-\frac tk}.
  \end{align}
  On the other hand, the Chernoff bound implies that $\Pr(|Z|\geq t)\leq 2\exp(-\nicefrac{t^2}{2})$.  With a union bound, with probability at least $1-\beta$,
\begin{align}
  \envert*{ \sqrt{2\sum_{j=1}^k Y_j}Z} &= \sqrt{2\sum_{j=1}^k Y_j}\envert{ Z}  \leq 2\sqrt{k}\ln\frac{2k}\beta. \qedhere
\end{align}
\end{proof}

Next, an $L^1$ (TV) convergence result of empirical distributions with finite support, which follows from the concentration of i.i.d.\ sum of Multinoulli random variables:

\begin{theorem}[\citealp{weissman2003InequalitiesL1Deviation}]
Let $p\in\Delta^{d-1}$ the $(d-1)$ simplex and $\hat p_n\sim\nicefrac1n\mathrm{Multinomial}(n,p)$, then with probability at least $1-\beta$, $\|p-\hat p_n\|_1\leq \sqrt{{\nicefrac{2d}n\ln\nicefrac2\beta}}$.
\end{theorem}

\begin{lemma}\label{lem:tv}
Let independent $x_{1},\dots,x_{n}\sim p$ with finite support $\calX$, and denote the empirical distribution by $\hat{p}_{n}=\nicefrac{1}{n}\sum_{i=1}^{n}\beta_{x_{i}}$, then with probability at least $1-\beta$, $\|p-\hat
p_n\|_1\leq\sqrt{\nicefrac{2|\calX|}n\ln\nicefrac 2\beta}$.
\end{lemma}

For the proof of \cref{thm:private.estimates}, we need two technical results:

\begin{lemma}\label{lem:1}
Let independent $x_1,\dots,x_n\sim p$ supported on $[k]$, and denote the empirical PMF by $\hat p_j=\nicefrac1n\sum_{i}\1[x_i=j]$, for all $j\in[k]$.  Let $E\subseteq[k]$ be a subset of size $\ell$, denote $w_E= \Pr(x\in E)$, the PMF conditioned on the event $x\in E$ by $p_{\mid E}$, and its empirical counterpart by $\hat p_{j\mid E}=\nicefrac1{n_E}\sum_{i}\1[x_i=j]$, where $n_E=\sum_{i}\1[x_j\in E]$.  Then for all $n\geq \nicefrac8{w_E}\ln\nicefrac8\beta$, with probability at least $1-\beta$,
  \begin{align}
    \enVert{p_{\mid E}-\hat p_{\mid E}}_\infty &\leq \sqrt{\frac{1}{nw_E}\ln\frac{8\ell}\beta}, \\
    \dTV(p_{\mid E},\hat p_{\mid E}) &=\frac12\enVert{p_{\mid E}-\hat p_{\mid E}}_1 \leq \sqrt{\frac {\ell}{4nw_E}\ln\frac {8}\beta}, \\
    \dKS(p_{\mid E},\hat p_{\mid E}) &\leq \sqrt{\frac{1}{nw_E}\ln\frac{8\ell}\beta}.
  \end{align}
\end{lemma}

\begin{proof}
By Chernoff bound on the Binomial distribution, for all $n\geq \nicefrac8{w_E}\ln\nicefrac2\beta$, with probability at least $1-\beta$,
\begin{equation}
  \frac{nw_E}{2}\leq n_E\leq2nw_E. \label{eq:emp.pmf.1}
\end{equation}

Order the samples so that the ones with $x_i\in E$ are at the front (or, consider first sample $n_E$, then sample the first $n_E$ samples from $p_{\mid E}$).  Conditioned on \cref{eq:emp.pmf.1}, by Hoeffding's inequality and a union bound, for all $j\in E$, with probability at least $1-\beta$,
  \begin{align}
    \envert{p_{j\mid E}-\hat p_{j\mid E}} 
    &= \envert*{\frac 1{w_E}\Pr(X= j) - \frac1{n_E}\sum_{i=1}^{n_E} \1[x_i= j]} \\
    &= \frac 1{w_E} \envert*{\Pr(X= j) - \frac{1}{n_E}\sum_{i=1}^{n_E} w_E \1[x_i= j]} \\
    &\leq \frac1{w_E} \sqrt{\frac{w_E^2}{2n_E}\ln\frac{2\ell}\beta} \\
    &\leq \sqrt{\frac{1}{nw_E}\ln\frac{2\ell}\beta}.
  \end{align}

  Next, by \cref{lem:tv}, with probability at least $1-\beta$,
  \begin{align}
    \dTV(p_{\mid E},\hat p_{\mid E}) &= \frac12\sum_{j\in E}\envert*{p_{j \mid E}-\hat p_{j\mid E}} \\
    &=  \frac12\sum_{j\in E} \envert*{ \frac1{w_E} P(X=j)  - \frac{1}{n_E}\sum_{i=1}^{n_E} \1[x_i=j] } \\
    &\leq  \frac12 \sqrt{\frac {\ell}{2n_E}\ln\frac {2}\beta} \\
    &\leq \sqrt{\frac {\ell}{4nw_E}\ln\frac {2}\beta}.
  \end{align}

  Lastly, $\dKS$ computes the $L^\infty$-distance between two CDFs, so similar to the $\ell_\infty$ bound, by Hoeffding's inequality and a union bound, with probability at least $1-\beta$,
  \begin{align}
    \dKS(p_{\mid E},\hat p_{\mid E}) 
    &= \max_{j\in E} \envert*{\frac1{w_E}\Pr(X\leq j) - \frac1{n_E}\sum_{i=1}^{n_E} \1[x_i\leq j]} \\
    &= \frac 1{w_E} \envert*{\Pr(X\leq j) - \frac{1}{n_E}\sum_{i=1}^{n_E} w_E \1[x_i\leq j]} \\
    &\leq \frac1{w_E} \sqrt{\frac{w_E^2}{2n_E}\ln\frac{2\ell}\beta} \\
    &\leq \sqrt{\frac{1}{nw_E}\ln\frac{2\ell}\beta}.
  \end{align}
  
  The result follows by taking a final union bound over the four events during the analysis and rescaling $\beta\gets \beta/4$.
\end{proof}

\begin{lemma}\label{lem:2}
  Let constants $a_1,\dots,a_k\geq0$, and independent $\xi_1,\dots,\xi_k\sim\mathrm{Laplace}(0,b)$. Denote $S_j= \sum_{\ell=1}^j a_j$, $s=S_k$, and let $t\geq0$.  Define for all $j\in[k]$,
  \begin{flalign}
  &\text{(add noise)}& x_j\coloneqq{}& a_j + \xi_j,& F_j={}&\sum_{k=1}^j x_k, & \\
  &\text{(isotonic regression)}&y_j={}& G_j - G_{j-1},& G_j\coloneqq{}& \frac12\rbr*{ F_{l_j}+F_{r_j}}, \\
  &\text{(clipping)}& z_j={}&H_j - H_{j-1}, & H_j\coloneqq{}&\begin{cases}
    \proj_{[0,t]}G_j & \text{if } j<k \\
    t & \text{else},
  \end{cases}
  \end{flalign}
  where
  \begin{equation}
    (l_j,r_j)=\argmax_{l\leq j\leq r}( F_{l}- F_{r}).
  \end{equation}
  Then with probability at least $1-\beta$,
  \begin{align}
    \enVert{a- z}_1 &\leq 3\envert{s-t} + 74bk\ln\frac{4k}\beta, \\
    \enVert{a- z}_\infty &\leq 2\envert{s-t} + 32b\sqrt {k}\ln\frac{4k}\beta, \\
    \enVert{S- H}_\infty &\leq \envert{s-t} + 12b\sqrt {k}\ln\frac{4k}\beta.
  \end{align}
\end{lemma}

\begin{proof}

Our analysis for proceeds by using the triangle inequality to decompose into and bounding each of the following terms (shown here for $\|\cdot\|_1$, analogously for $\|\cdot\|_\infty$ and the partial sums):
\begin{equation}\label{eq:pmf.steps}
  \enVert{a - z}_1\leq   \enVert{a - x}_1 +   \enVert{x - y}_1 + \enVert{y - z}_1.
\end{equation}

We will use the following concentration result of Laplace random variables: by \cref{cor:laplace}, with probability at least $1-\beta$, for all $0\leq \ell\leq m\leq k$,
\begin{equation}
  \envert*{  \sum_{j=\ell+1}^{m}\xi_j} \leq 4b\sqrt{m-\ell}\ln\frac{2(m-\ell)k^2}\beta \leq 12b\sqrt {m-\ell}\ln\frac{2k}\beta. \label{eq:sum.bounds}
\end{equation}
  
\paragraph{First Term in \cref{eq:pmf.steps}.}
  By the CDF of the exponential distribution (of which the Laplace distribution is the two-sided version), with a union bound, with probability at least $1-\beta$,
  \begin{equation}
    \envert{a_j-x_j} = \envert{\xi_j} \leq 2b\ln\frac k\beta,\quad\forall j\in[k],
  \end{equation}
  and it follows that
  \begin{equation}
    \enVert{a- x}_1 = \sum_{j=1}^k \envert{a_j-x_j} \leq 2bk\ln\frac k\beta.
  \end{equation}
  
  For the partial sums, by \cref{eq:sum.bounds},
  \begin{align}
    \|S-H\|_\infty = \max_j \envert*{S_j- F_j} = \max_j \envert*{\sum_{\ell=1}^j \xi_\ell} \leq \max_j 12b\sqrt {j}\ln\frac{2k}\beta =12b\sqrt {k}\ln\frac{2k}\beta.
  \end{align}

\paragraph{Second Term in \cref{eq:pmf.steps}.}
  
  Note that for any $\ell\leq m$ such that $F_{\ell} \geq F_{m}$ (i.e., a violating pair for isotonic regression),
  \begin{equation}
    G_\ell - G_m = F_\ell - F_\ell - \sum_{j=\ell+1}^m  a_\ell - \sum_{j=\ell+1}^m \xi_\ell  \leq -\sum_{j=\ell+1}^m \xi_\ell \leq 12b\sqrt {m-\ell}\ln\frac{2k}\beta \label{eq:pmf.3},
  \end{equation}
  because $a_j\geq 0$.
  So for all $j\in[k]$,
  \begin{align}
    0\leq \envert{x_j-y_j} 
    &=\envert*{ G_{j}- G_{j-1}- (F_{j}-F_{j-1}) } \\
    &\leq\envert*{ G_{j}-F_{j}} +\envert*{ G_{j-1}- F_{j-1} } \\
    &=   \begin{rcases}
          \begin{dcases}
            G_{j}-F_{j} & \text{if } G_{j} > F_{j} \\
            F_{j} - G_{j} &\text{else}
          \end{dcases}
        \end{rcases} + \begin{rcases}
          \begin{dcases}
            G_{j-1}-F_{j-1} & \text{if } G_{j-1} > F_{j-1} \\
            F_{j-1} - G_{j-1} &\text{else}
          \end{dcases}
        \end{rcases} \\
    &=   \begin{rcases}
          \begin{dcases}
            G_{j}-\frac{G_{l_j} + G_{r_j}}2 & \text{if } G_{j} > F_{j} \\
            \frac{G_{l_j} + G_{r_j}}2 - G_{j} &\text{else}
          \end{dcases}
        \end{rcases} +  \begin{rcases}
          \begin{dcases}
            G_{j-1}-\frac{G_{l_{j-1}} + G_{r_{j-1}}}2 & \text{if } G_{j-1} > F_{j-1} \\
            \frac{G_{l_{j-1}} + G_{r_{j-1}}}2 - G_{j-1} &\text{else}
          \end{dcases}
        \end{rcases} \\
      &\leq \begin{rcases}
          \begin{dcases}
            G_{j}-\frac{G_j + G_{r_j}}2 & \text{if } G_{j} > F_{j} \\
            \frac{G_{l_j} + G_j}2 - G_{j} &\text{else}
          \end{dcases}
        \end{rcases} +  \begin{rcases}
          \begin{dcases}
            G_{j-1}-\frac{G_{j-1} + G_{r_{j-1}}}2 & \text{if } G_{j-1} > F_{j-1} \\
            \frac{G_{l_{j-1}} + G_{j-1}}2 - G_{j-1} &\text{else}
          \end{dcases}
        \end{rcases} \\
      &= \frac12 \begin{rcases}
          \begin{dcases}
            G_j - G_{r_j} & \text{if } G_{j} > F_{j} \\
            G_{l_j} - G_j &\text{else}
          \end{dcases}
        \end{rcases} +  \frac12 \begin{rcases}
          \begin{dcases}
            G_{j-1} - G_{r_{j-1}} & \text{if } G_{j-1} > F_{j-1} \\
            G_{l_{j-1}} - G_{j-1} &\text{else}
          \end{dcases}
        \end{rcases} \\
        &\leq 12b\sqrt{\max(r_j-j,j-l_j)}\ln\frac{2k}\beta 
        \leq 12b\sqrt{\max(k-j,j)}\ln\frac{2k}\beta,
  \end{align}
where line 5 is because $G_{r_m}\leq G_{m}\leq G_{l_m}$ for all $j$, and line 7 by \cref{eq:pmf.3}.  It then follows that
\begin{equation}
  \enVert{x-y}_1 = \sum_{j=1}^k \envert{x_j-y_j} \leq 24bk\ln\frac{2k}\beta.
\end{equation}

Lastly, using the fact that the $L^\infty$ error of $L^\infty$ isotonic regression is $\nicefrac12\max_{\ell\leq m}(G_\ell-G_m)$~\citep{stout2017InfinityIsotonicRegression}, and by \cref{eq:pmf.3} again,
\begin{align}
  0\leq \enVert{F-G}_\infty  = \frac12\max_{\ell\leq m} (G_\ell - G_m) \leq 6 b\sqrt{k}\ln\frac{2k}\beta. 
\end{align}

\paragraph{Third Term in \cref{eq:pmf.steps}.}

  Note that because $a_j\in[0,s]$, by \cref{eq:sum.bounds},
  \begin{align}
    \min_j G_j&= \min_j \frac12(F_{l_j}+F_{r_j}) \geq  \min_j F_{j}   = \min_j \sum_{m=1}^{j}\rbr*{a_m+\xi_m} \\
    &\geq \min_j \sum_{m=1}^{j}\xi_m \geq -\max_j\envert*{\sum_{j=1}^k\xi_j}\geq -12b\sqrt {k}\ln\frac{2k}\beta, \label{eq:pmf.4}
  \end{align}
  and similarly,
  \begin{equation}
    \max_j G_j \leq \max_j \sum_{m=1}^{j}\rbr*{a_m+\xi_m} \leq s + 12b\sqrt {k}\ln\frac{2k}\beta.\label{eq:pmf.5}
  \end{equation}

  Since $G$ is nondecreasing after isotonic regression, clipping only affects its prefix and/or suffix.
  For the prefix, let ${l}=\max\{j\in[k]:H_j=0\}$. If ${l}$ does not exist, then no clipping to zero has occurred.  Otherwise, for all $j\leq {l}$, by \cref{eq:pmf.4},
  \begin{equation}
  \envert{G_j-H_j} = \max(-G_j, 0) \leq 12b\sqrt {k}\ln\frac{2k}\beta,
\end{equation}
and
    \begin{equation}
      \envert{y_j-z_j} = \envert{G_j-G_{j-1} - (H_j - H_{j-1})} \leq - G_j - G_{j-1} \leq 2\max\rbr*{-\min_j G_j,0} \leq 24b\sqrt {k}\ln\frac{2k}\beta.
    \end{equation}
  For the suffix, let $r=\min\{j\in[k]:H_j=t\}$, then for all $r\leq j<k$, by \cref{eq:pmf.5},
    \begin{equation}
      \envert{G_j-H_j} = G_j - t \leq \envert{s - t} + 12b\sqrt {k}\ln\frac{2k}\beta,
    \end{equation}
    and
    \begin{equation}
      \envert{y_j-z_j} \leq \rbr*{G_j - t} + \max(G_{j-1} - t,0)  \leq 2\max\rbr*{\max_j G_j-t,0} \leq 2\envert{s-t} + 24b\sqrt {k}\ln\frac{2k}\beta;
    \end{equation}    
    for $j=k$,
    \begin{align}
       \envert{G_k-H_k}  \leq \begin{cases}
        G_k - t & \text{if } G_k \geq t \\
        t- G_k  & \text{else}
      \end{cases}
      = \begin{cases}
        G_k - t & \text{if } G_k \geq t \\
        t-\rbr*{s + \sum_{m=1}^k\xi_m}  & \text{else}
      \end{cases}
       \leq \envert{s-t} + 12b\sqrt {k}\ln\frac{2k}\beta,
    \end{align}
    and
    \begin{align}
      \envert{y_j-z_j} &\leq \envert*{G_j - t} + \max(G_{j-1} - t,0) \leq 2\envert{s-t} + 24b\sqrt {k}\ln\frac{2k}\beta.
    \end{align}    

  Finally, for $\|\cdot\|_1$,
  \begin{align}
    \enVert*{y-z}_1 &= \sum_{j=1}^{l} \envert{G_j-G_{j-1} - (H_j - H_{j-1})} + \sum_{j=r}^k\envert{G_j-G_{j-1} - (H_j - H_{j-1})} \\
    &= \sum_{j=1}^{l} \rbr{G_j-G_{j-1}} + \envert{y_r-z_r} + \sum_{j=r+1}^{k-1}\rbr{G_j-G_{j-1}} + \envert{y_k-z_k} \\
    &= G_{l}-G_{1} + \envert{y_r-z_r} + G_{k-1}-G_{r} + \envert{y_k-z_k} \\
    &\leq -G_{1} + \envert{y_r-z_r} + G_{k-1}-t + \envert{y_k-z_k} \\
    &\leq 12b\sqrt {k}\ln\frac{2k}\beta + 2\rbr*{\envert{s-t} + 12b\sqrt {k}\ln\frac{2k}\beta} + \rbr*{s + 12b\sqrt {k}\ln\frac{2k}\beta}-t  \\
    &\leq 3\envert{s-t} + 48b\sqrt {k}\ln\frac{2k}\beta,
  \end{align}
  keep in mind that $1\leq l$ and $r\leq k-1,k$ on line 3 and onward.
  
  The result follows by taking a final union bound over the two events above and rescaling $\beta\gets \beta/2$.
\end{proof}

\begin{proof}[Proof of \cref{thm:private.estimates}]
  Because $w_a\geq 0$, by \cref{cor:laplace}, with probability at least $1-\beta$,
  \begin{align}
    \envert{w_a-\tilde w_a} 
    &= \envert*{w_a - \max\rbr*{w_a + \sum_{j=1}^k\mathrm{Laplace}(0,\nicefrac2{n\varepsilon}), 0}} \\
    &= \envert*{\max\rbr*{\sum_{j=1}^k\mathrm{Laplace}(0,\nicefrac2{n\varepsilon}), -w_a}}  \\
    &\leq \envert*{\sum_{j=1}^k\mathrm{Laplace}(0,\nicefrac2{n\varepsilon})} \\
    &\leq O\rbr*{\frac{\sqrt{k}}{n\varepsilon}\ln\frac{k|\calA|}{\beta}}. \label{eq:main.1}
  \end{align}

  Next, define
  \begin{align}
    \hat p_a(v_j) = \frac{1}{n_a}\sum_{i=1}^n\1[h\circ f(x_i,a_i), a_i=a], \quad  \check p_a(v_j) = \frac{1}{\tilde w_a}\tilde p(a,v_j),
  \end{align}
  where $n_a=\nicefrac{1}{n}\sum_{i=1}^n\1[a_i=a]$.  Note that $\widecheck F_a(v_j)=\sum_{\ell=1}^j\check p_a(v_\ell)$.  By triangle inequality,
  \begin{align}
    \MoveEqLeft \enVert{p_a - \tilde p_a}_\infty \\
    &\leq \enVert{\tilde p_a - \check p_a}_\infty + \enVert{\check p_a - \hat p_a}_\infty + \enVert{\hat p_a - p_a}_\infty \\
    &= \enVert*{\tilde p_a - \frac1{\tilde w_a}\check p(a,\cdot)}_\infty + \enVert*{\frac1{\tilde w_a}\check p(a,\cdot) - \frac1{\hat w_a}\hat p(a,\cdot)}_\infty + \enVert{\hat p_a - p_a}_\infty \\
    &\leq \enVert*{\tilde p_a - \frac1{\hat w_a}\check p(a,\cdot)}_\infty + 2\enVert*{\frac1{\tilde w_a}\check p(a,\cdot) - \frac1{\hat w_a}\check p(a,\cdot)}_\infty + \enVert*{\frac1{\hat w_a}\check p(a,\cdot) - \frac1{\hat w_a}\hat p(a,\cdot)}_\infty + \enVert{\hat p_a - p_a}_\infty \\
    &= \frac1{\hat w_a}\enVert*{\hat w_a \tilde p_a - \check p(a,\cdot)}_\infty + \frac2{\hat w_a}\envert*{\hat w_a - \tilde w_a} + \frac1{\hat w_a}\enVert{\check p(a,\cdot) - \hat p(a,\cdot)}_\infty + \enVert{\hat p_a - p_a}_\infty\\
    &\leq \begin{multlined}[t][5.5in]
        \frac1{\hat w_a}\enVert*{\tilde w_a \tilde p_a - \hat w_a \tilde p_a}_\infty + \frac1{\hat w_a}\enVert*{\tilde w_a \tilde p_a - \check p(a,\cdot)}_\infty + \frac2{\hat w_a}\envert*{\hat w_a - \tilde w_a} \\ + \frac1{\hat w_a}\enVert{\check p(a,\cdot) - \hat p(a,\cdot)}_\infty + \enVert{\hat p_a - p_a}_\infty
    \end{multlined}\\
    &\leq  \frac1{\hat w_a}\enVert*{\tilde w_a \tilde p_a - \check p(a,\cdot)}_\infty + \frac3{\hat w_a}\envert*{\hat w_a - \tilde w_a} + \frac1{\hat w_a}\enVert{\check p(a,\cdot) - \hat p(a,\cdot)}_\infty + \enVert{\hat p_a - p_a}_\infty\\
    &\leq O\rbr*{  \frac1{\hat w_a} \rbr*{\envert*{\hat w_a - \tilde w_a} + \frac{\sqrt k}{n\varepsilon}\ln\frac{k|\calA|}{\beta}+  \frac{1}{n\varepsilon}\ln\frac{k|\calA|}{\beta}} +  \sqrt{\frac{1}{nw_a}\ln\frac{k|\calA|}\beta} }  \\
    &\leq O\rbr*{ \frac1{\hat w_a} \rbr*{\frac{\sqrt{k}}{n\varepsilon}\ln\frac{k|\calA|}{\beta} + \frac{\sqrt k}{n\varepsilon}\ln\frac{k|\calA|}{\beta}+  \frac{1}{n\varepsilon}\ln\frac{k|\calA|}{\beta}} +  \sqrt{\frac{1}{nw_a}\ln\frac{k|\calA|}\beta} }  \\
    &\leq O\rbr*{ \frac{\sqrt{k}}{n\hat w_a\varepsilon}\ln\frac{k|\calA|}{\beta} +  \sqrt{\frac{1}{nw_a}\ln\frac{k|\calA|}\beta}  }\\
    &\leq O\rbr*{ \frac{\sqrt{k}}{n w_a\varepsilon}\ln\frac{k|\calA|}{\beta} +  \sqrt{\frac{1}{nw_a}\ln\frac{k|\calA|}\beta}  }
  \end{align}
  by \cref{eq:main.1,lem:1,lem:2}; the bound on the $\enVert{\check p(a,\cdot) - \hat p(a,\cdot)}_\infty$ follows the same analysis in the proof of \cref{lem:2} for the first term.

  Similarly,
  \begin{align}
    \enVert{p_a - \tilde p_a}_1 
    &\leq  \frac1{\hat w_a}\enVert*{\tilde w_a \tilde p_a - \check p(a,\cdot)}_1 + \frac3{\hat w_a}\envert*{\hat w_a - \tilde w_a} + \frac1{\hat w_a}\enVert{\check p(a,\cdot) - \hat p(a,\cdot)}_1 + \enVert{\hat p_a - p_a}_1\\
    &\leq O\rbr*{ \frac{{k}}{n w_a\varepsilon}\ln\frac{k|\calA|}{\beta} +  \sqrt{\frac{k}{nw_a}\ln\frac{|\calA|}\beta}  },
  \end{align}
  and
  \begin{align}
    \dKS(p_a, \tilde p_a) 
    &\leq \begin{multlined}[t][4.85in]
        \frac1{\hat w_a}\max_j\envert*{\sum_{\ell=1}^j \rbr*{  \tilde w_a \tilde p_a(v_\ell) - \check p(a,v_\ell) } } + \frac3{\hat w_a}\envert*{\hat w_a - \tilde w_a} \\+ \frac1{\hat w_a}\max_j\envert*{\sum_{\ell=1}^j \rbr*{\check p(a,\cdot) - \hat p(a,\cdot)}} + \dKS(\hat p_a, p_a)
    \end{multlined}\\
    &\leq O\rbr*{ \frac{\sqrt{k}}{n w_a\varepsilon}\ln\frac{k|\calA|}{\beta} +  \sqrt{\frac{k}{nw_a}\ln\frac{k|\calA|}\beta}  }.\qedhere
  \end{align}
\end{proof}

For the proof of \cref{thm:main}, we need the following technical result for the difference of $W_2^2$ distances:

\begin{lemma}[\citealp{chizat2020FasterWassersteinDistance}]\label{lem:w22.diff}
  Let $\mu,\mu',\nu,\nu'$ be distributions whose supports are contained in the centered ball of radius $R$ in $\RR^d$, then
  \begin{align}
    \MoveEqLeft\envert*{W_2^2(\mu,\nu) - W_2^2(\mu',\nu')} \\
    &\leq \begin{multlined}[t]
  \envert*{\int \|x\|_2^2 \dif(\mu-\mu')(x)} + \envert*{\int \|x\|_2^2 \dif(\nu-\nu')(x)} \\ + 2R\sup_{\textnormal{convex } f\in\mathrm{Lip}(1)}\envert*{\int f(x) \dif(\mu-\mu')(x)} + 2R\sup_{\textnormal{convex } g\in\mathrm{Lip}(1)}\envert*{\int g(x) \dif(\nu-\nu')(x)}
\end{multlined} \\
&\leq 4RW_1(\mu,\mu') + 4RW_1(\nu,\nu').
  \end{align}
\end{lemma}

The last line follows from the dual representation of $W_1$ distance for distributions with bounded support:
\begin{equation}
  W_1(\mu,\nu) = \sup_{f\in\mathrm{Lip}(1)}\envert*{\int f(x) \dif(\mu-\nu)(x)},
\end{equation}
and the fact that $x\mapsto \|x\|_2^2$ is $2R$-Lipschitz on the centered ball of radius $R$.

Also, recall the fact that the $W_1$ distance of distributions supported on a ball of radius $R$ can be upper bounded by total variation distance:
\begin{align}
    W_1(\mu,\nu)&=\inf_{\pi\in\Pi(\mu,\nu)} \int d(x,y) \dif \pi(x,y) \\
    &\leq 2R\inf_{\pi\in\Pi(\mu,\nu)} \int \1[x\neq y] \dif \pi(x,y) \\
    &= 2R\rbr*{1-\sup_{\pi\in\Pi(\mu,\nu)} \int \1[x= y] \dif \pi(x,y)} \\
    &= 2R\rbr*{1- \int \min\rbr*{\mu(x), \nu(x)} \dif x}\\
    &= 2R{\int \max\rbr*{0, \nu(x) - \mu(x)} \dif x} \\
    &= R{\int  \envert*{ \nu(x) - \mu(x) } \dif x} \\
    &\eqqcolon R \|\mu - \nu\|_1 \\
    &\eqqcolon 2R \dTV(\mu,\nu), \label{eq:w1.01loss.1}
\end{align}
where line 6 is because $\int \rbr{ \mu(x) - \nu(x) } \dif x = 0$.

And, note the following simple fact regarding optimal transports $T^*:\RR\rightarrow\RR$ under the squared cost; in the special case where $T^*$ is a Monge transportation plan, the lemma is equivalent to saying that $T^*$ is a nondecreasing function (see the last panel of \cref{fig:steps} for a picture):

\begin{lemma}\label{lem:transport.property}
  Let $p,q$ be two distributions supported on $[k]$, and $\pi\in\argmin_{\pi'\in\Pi(p,q)}\sum_{m,\ell} (m-\ell)^2 \pi(m,\ell)$, then for all $j\in[k]$,
  \begin{equation}
    f(m)\coloneqq \sum_{\ell=1}^j\frac{\pi(m,\ell)}{p(m)} \begin{cases}
      = 1 & \text{if } m < l_j \\
      \in[0,1] & \text{if } m = l_j \\
      =0 & \text{if } m > l_j, \\
    \end{cases} \quad \forall m \textrm{ s.t. } p(m)>0, \label{eq:transport.property}
  \end{equation}
  for some $l_1,\dots,l_k\in[k]$.
\end{lemma}

\begin{proof}
  Let $j\in[k]$ be arbitrary.  Suppose to the contrary that $\nexists l_j$ such that \cref{eq:transport.property} holds, then either $f$ as a function of $m$ is not non-increasing, i.e., $f(m+1)>f(m)$ for some $m$, or $0<f(m+1)\leq f(m)<1$.  We show that either of these contradicts the optimality of $\pi$.
  
  In both cases, there must exists $l\leq j<r$ such that $\pi(m,r),\pi(m+1,l)\geq q$ for some $q> 0$, because $\sum_{\ell=1}^j\pi(m+1,\ell)=p(m+1)f(m+1)>0$ and $\sum_{\ell=j+1}^k\pi(m,\ell)=p(m)(1-f(m))>0$.  Then a coupling $\gamma$ with a lower cost can be constructed by (partially) exchanging the two entries:
  \begin{equation}
    \gamma(i,\ell)= \begin{cases}
      \pi(m,r) - q  &\text{if } i= m, \ell = r \\
      \pi(m,l) + q  &\text{if } i= m, \ell = l \\
      \pi(m+1,r) + q  &\text{if } i= m+1, \ell = r \\
      \pi(m+1,l) - q  &\text{if } i= m+1, \ell = l \\
      \pi(i,\ell) &\text{else}.\\
    \end{cases}
  \end{equation}
  We verify that it has a lower cost than $\pi$:
  \begin{align}
    \sum_{i,\ell}(i-\ell)^2 (\gamma-\pi)(i,\ell)
    &=  -q(m-r)^2 + q(m-l)^2 + q(m+1-r)^2 - q(m+1-l)^2 
    \\&= 2q(l-r) \\&< 0. \qedhere
  \end{align}
\end{proof}

\begin{proof}[Proof of \cref{thm:main}]

This proof also relies on properly applying the triangle inequality to decompose into comparable terms.  We list the terms that will be compared here:
\begin{itemize}
  \item Denote the Bayes regressor by $f^*$, and recall that $f$ is the regressor being post-processed.  Denote the output distribution of $f^*$ conditioned on group $a$ by $r^*_a\coloneqq f^*(\cdot,a)\sharp\mu_{X\mid a}$, which will be compared to that of $f$, $r_a\coloneqq f(\cdot,a)\sharp\mu_{X\mid a}$.

  \item Given a discretizer $h$, the discretized conditional output distribution of $f^*$ is denoted by $p^*_a\coloneqq h\sharp r^*_a$, and that of $f$ by $p_a\coloneqq h\sharp r_a$.  We will compare $r^*_a$ to its discretized version $p^*_a$, and $p_a$ to $\tilde p_a$, the empirical conditional discretized output distributions of $f$ estimated privately.
  
  \item Denote the private group marginal distribution estimated from the samples by $\tilde w_a$, which will be compared to the ground-truth $w_a\coloneqq \Pr(A=a)$.

  \item Let $(\cdot,\{\tilde q'_a\}_{a\in\calA})\gets \Prob(\{\tilde p_a\}_{a\in\calA},\{\tilde w_a\}_{a\in\calA},\alpha)$ and $(\cdot,\cdot,\{\tilde q_a\}_{a\in\calA})\gets \LProb(\{\tilde p_a\}_{a\in\calA},\{\tilde w_a\}_{a\in\calA},\alpha)$.  The difference between $\tilde q'_a,\tilde q_a$ is that the support of the latter is restricted to $v$.

  Recall that the fair regressor returned from \cref{alg:post.proc} has the form $\bar f=g_a\circ h\circ f$, where $g_a$ is the optimal transport from $\tilde p_a$ to $\tilde q_a$.  The $\tilde q_a$'s will be compared to the $\tilde q'_a$'s, which are in turn compared to the output distributions $q^*_a$ of an optimal $\alpha$-fair regressor, denoted by $\bar f^*$ (note that $(\cdot,\{q^*_a\}_{a\in\calA})\gets \Prob(\{r^*_a\}_{a\in\calA},\{w_a\}_{a\in\calA},\alpha)$).
\end{itemize}

\paragraph{Error Bound.}
Note that $R(\bar f)-R(\bar f^*) = \ER(\bar f)-\ER(\bar f^*)$.  We begin with analyzing the first term.   By the orthogonality principle,
\begin{align}
  \ER(\bar f) 
  &= \E\sbr*{\rbr*{ \bar f(X,A) - f^*(X,A) }^2}\\
  &= \sum_{a\in\calA} w_a \E_{X\sim\mu_{X\mid a}} \sbr*{\rbr*{ \bar f(X,a) - f^*(X,a) }^2} \\
  &= \sum_{a\in\calA} w_a \E_{X\sim\mu_{X\mid a}} \sbr*{\rbr*{ \bar f(X,a) - f(X,a) + \rbr*{f(X,a) - f^*(X,a)} }^2} \\
  &= \begin{multlined}[t][5in]  \sum_{a\in\calA} w_a \Big( \E_{X\sim\mu_{X\mid a}} \Big[ \rbr*{ \bar f(X,a) - f(X,a) }^2 + \rbr*{ f(X,a) - f^*(X,a) }^2  \\[-0.5em]
  + 2\rbr*{ \bar f(X,a) - f(X,a) }\rbr*{ f(X,a) - f^*(X,a) } \Big]\Big)
  \end{multlined} \\
  &\leq \sum_{a\in\calA} w_a \E_{X\sim\mu_{X\mid a}} \sbr*{\rbr*{ \bar f(X,a) - f(X,a)}^2} + \underbrace{ 3\E\sbr*{\envert*{f(X,A) - f^*(X,A)} }}_{\calE_{1}},
\end{align}
where line 5 is because of the assumption that the images of $\bar f,f$ are contained in $[0,1]$.  The second term on the last line is the $L^1$ excess risk of $f$; for the first term,
\begin{align}
  \MoveEqLeft \E_{X\sim\mu_{X\mid a}} \sbr*{\rbr*{ \bar f(X,a) - f(X,a)}^2} \\
  &= \E_{X\sim\mu_{X\mid a}} \sbr*{\rbr*{ g_a\circ h\circ  f(X,a) - f(X,a)}^2} \\
  &= \E_{X\sim\mu_{X\mid a}} \sbr*{\rbr*{ g_a\circ h\circ  f(X,a) - h\circ f(X,a) + \rbr*{ h\circ  f(X,a) - f(X,a) }}^2} \\
  &\leq \E_{X\sim\mu_{X\mid a}} \sbr*{\rbr*{ g_a\circ h\circ  f(X,a) - h\circ f(X,a)}^2 + 3\envert*{h\circ  f(X,a) - f(X,a)} } \\
  &\leq  \E_{X\sim\mu_{X\mid a}} \sbr*{\rbr*{ g_a\circ h\circ  f(X,a) - h\circ f(X,a)}^2 }  + \frac{3}{2k} \\
  &= \sum_{j=1}^k  p_a(v_j) \E \sbr*{\rbr*{ g_a(v_j) - v_j}^2 }  + \frac{3}{2k} \\
  &= \sum_{j=1}^k  \tilde p_a(v_j) \E \sbr*{\rbr*{ g_a(v_j) - v_j}^2 } + \sum_{j=1}^k ( p_a(v_j) -  \tilde p_a(v_j)) \E \sbr*{\rbr*{ g_a(v_j) - v_j}^2 }  + \frac{3}{2k} \\
  &\leq  W_2^2(\tilde p_a, \tilde q_a) + \enVert*{  p_a -  \tilde p_a}_1  + \frac{3}{2k} \\
  &\leq  W_2^2(\tilde p_a, \tilde q_a) +  O\rbr*{\sqrt{\frac{k}{nw_a}\ln\frac{k|\calA|}\beta} + \frac{k}{n w_a\varepsilon}\ln\frac{k|\calA|}{\beta}  }   + \frac{3}{2k},
\end{align}
where line 4 is because $h$ discretizes the input to the midpoint of the bin that it falls in, which displaces it by up to $\nicefrac L{2k}=\nicefrac1{2k}$; line 5 is because $p_a(v_j) = \Pr(h\circ f(X,a)=v_j)$; the first term on line 7 is because $g_a$ is the optimal transport from $\tilde p_a$ to $\tilde q_a$.  Then, combining the above, by \cref{thm:private.estimates}, with probability at least $1-\beta$,
\begin{align}
  \ER(\bar f) 
  &\leq \sum_{a\in\calA} \rbr*{ w_a W_2^2(\tilde p_a, \tilde q_a) + O\rbr*{ \sqrt{\frac{kw_a}{n}\ln\frac{k|\calA|}\beta}  +  \frac{k}{n\varepsilon}\ln\frac{k|\calA|}{\beta} }}   + \calE_1  + \frac{3}{2k} \\
  &\leq \sum_{a\in\calA}  w_a W_2^2(\tilde p_a, \tilde q_a) + \calE_1 + \underbrace{O\rbr*{ \sqrt{\frac{k|\calA|}{n}\ln\frac{k|\calA|}\beta}  +  \frac{k|\calA|}{n\varepsilon}\ln\frac{k|\calA|}{\beta} }}_{\calE_{2}}    + \frac{3}{2k} \\
  &\leq \sum_{a\in\calA} \rbr*{\tilde w_a W_2^2(\tilde p_a, \tilde q_a) + \envert*{\tilde w_a - w_a}W_2^2(\tilde p_a, \tilde q_a)} + \calE_1+ \calE_{2}    + \frac{3}{2k}  \\
  &\leq \sum_{a\in\calA} \rbr*{ \tilde w_a W_2^2(\tilde p_a, \tilde q_a) + \envert*{\tilde w_a - w_a}}  + \calE_1+ \calE_{2}    + \frac{3}{2k} \\
  &\leq \sum_{a\in\calA} \tilde w_a W_2^2(\tilde p_a, \tilde q_a) + \calE_1+ \calE_{2}    + \frac{3}{2k} \\
  &\leq \sum_{a\in\calA} \tilde w_a W_2^2(\tilde p_a, h\sharp \tilde q'_a) + \calE_{1}  + \calE_{2}  + \frac{3}{2k} \\
  &\leq \sum_{a\in\calA} \tilde w_a \rbr*{W_2(\tilde p_a, \tilde q'_a) + W_2(\tilde q'_a, h\sharp\tilde q'_a)}^2 + \calE_{1}  + \calE_{2}   + \frac{3}{2k} \\
  &= \sum_{a\in\calA} \tilde w_a \rbr*{W^2_2(\tilde p_a, \tilde q'_a) + 2W_2(\tilde p_a, \tilde q'_a) W_2(\tilde q'_a, h\sharp\tilde q'_a) + W^2_2(\tilde q'_a, h\sharp\tilde q'_a)} + \calE_{1}  + \calE_{2}   + \frac{3}{2k} \\
  &\leq \sum_{a\in\calA} \tilde w_a W_2^2(\tilde p_a, \tilde q'_a) + \calE_{1}  + \calE_{2}   + \frac{3}{k} \\
  &\leq \sum_{a\in\calA} \tilde w_a W_2^2(\tilde p_a, q^*_a) + \calE_{1}  + \calE_{2}   + \frac{3}{k},
\end{align}
where line 6 follows by noting that $\{h\sharp \tilde q'_a\}_{a\in\calA}$ is a feasible solution to $\LProb$, as it can be verified that $\dKS(h\sharp \tilde q'_a,h\sharp \tilde q'_{a'})\leq\alpha$ given that $\dKS(\tilde q'_a,\tilde q'_{a'})\leq\alpha$, $\forall a,a'\in\calA$ (hence restricting the support of the barycenter to $v$ introduces an additional error of $\nicefrac3{2k}$ as discussed in footnote~\ref{fn:restrict}), and the last line is because $\{\tilde q'_a\}_{a\in\calA}$ is a minimizer of $\Prob(\{\tilde p_a\}_{a\in\calA},\{\tilde w_a\}_{a\in\calA},\alpha)$.

So for the suboptimality of $\bar h$, by \cref{thm:opt.fair.equiv}, 
\begin{align}
  \ER(\bar f) - \ER(\bar f^*) 
  &\leq \sum_{a\in\calA} \tilde w_a \rbr*{ W_2^2(\tilde p_a, q^*_a) - W_2^2( r^*_a, q^*_a)} + \calE_{1} + \calE_{2}   + \frac{3}{k} \\
  &\leq \sum_{a\in\calA} \tilde w_a \rbr*{ W_2^2(\tilde p_a, q^*_a) - \rbr*{W_2(q^*_a, p^*_a) - W_2( r^*_a, p^*_a)}}^2 + \calE_{1} + \calE_{2}   + \frac{3}{k}\\
  &\leq \sum_{a\in\calA} \tilde w_a \rbr*{ W_2^2(\tilde p_a, q^*_a) - W_2^2( p^*_a, q^*_a) + 2W_2( p^*_a, q^*_a)W_2( p^*_a, r^*_a) } + \calE_{1} + \calE_{2}   + \frac{3}{k}\\
  &\leq \sum_{a\in\calA} \tilde w_a \rbr*{ W_2^2(\tilde p_a, q^*_a) - W_2^2( p^*_a, q^*_a)  } +\frac1k  + \calE_{1} + \calE_{2}   + \frac{3}{k},
\end{align}
where the last line is because $h$ is a transport from $r^*_a$ to $p^*_a$ with displacements of at most $\nicefrac{1}{2k}$;  for the first term, by \cref{lem:w22.diff,eq:w1.01loss.1},
\begin{align}
  \MoveEqLeft\sum_{a\in\calA}\tilde w_a \rbr*{ W_2^2(\tilde p_a, q^*_a) - W_2^2( p^*_a, q^*_a) }\\
  &\leq 4 \sum_{a\in\calA}   \tilde w_a W_1(\tilde p_a,p^*_a) \\
  &\leq 4 \sum_{a\in\calA}   \tilde w_a \rbr*{W_1(\tilde p_a,p_a) + W_1(p_a,r_a) + W_1(r_a,r^*_a) + W_1(r^*_a,p^*_a) } \\
  &\leq 4 \sum_{a\in\calA}   \tilde w_a \rbr*{\enVert*{\tilde p_a - p_a}_1 + \frac{1}{2k} + \E_{X\sim\mu_{X\mid a}}\sbr*{\envert*{f^*(X,a) - f(X,a)}} + \frac{1}{2k} }\\
  &\leq 4 \sum_{a\in\calA}   \tilde w_a \E_{X\sim\mu_{X\mid a}}\sbr*{\envert*{f^*(X,a) - f(X,a)}} + \calE_2 + \frac{4}{k} \\
  &\leq 4 \sum_{a\in\calA}   \rbr*{w_a - w_a + \tilde w_a } \E_{X\sim\mu_{X\mid a}}\sbr*{\envert*{f^*(X,a) - f(X,a)}} + \calE_2 + \frac{4}{k} \\
  &\leq 4  \sum_{a\in\calA}  w_a \E_{X\sim\mu_{X\mid a}}\sbr*{\envert*{f^*(X,a) - f(X,a)}} + 4 \sum_{a\in\calA}   \envert*{\tilde w_a - w_a} + 4\calE_2 + \frac{4}{k} \\
  &\leq 4 \calE_1 + \calE_2 + \frac{4}{k},
\end{align}
the third inequality is because the joint distribution of $(f(X,a),f^*(X,a))$ is a valid coupling belonging to $\Pi(r_a, r^*_a)$ that incurs a transportation cost of
\begin{align}
  \E\sbr{\envert{f^*(X,a) - f(X,a)}}  = \int \envert*{y-y^*} \dif \Pr(f(X,a)=y, f^*(X,a)=y^*) \geq W_1(r_a,r^*_a).
\end{align}

Putting everything together, and with a union bound over the two events above, gives the result in the theorem statement.

\paragraph{Fairness Guarantee.}
Let $\bar p_a$ denote the output distribution of the post-processed regressor $\bar f$ conditioned on group $a$.  Using triangle inequality, for any $a,a'\in\calA$,
\begin{equation}
  \dKS(\bar p_a, \bar p_{a'}) \leq \dKS( \tilde q_{a}, \tilde q_{a'}) + \dKS(\bar p_a, \tilde q_{a}) + \dKS(\bar p_{a'}, \tilde q_{a'})\leq \alpha + \dKS(\bar p_a, \tilde q_{a}) + \dKS(\bar p_{a'}, \tilde q_{a'}),
\end{equation}
where
\begin{align}
  \dKS(\bar p_{a}, \tilde q_{a}) 
  &= \max_j  \envert*{\sum_{\ell=1}^j \rbr*{ \bar p_a(v_\ell) - \tilde q_a(v_\ell) }} \\
  &= \max_j \envert*{\sum_{\ell=1}^j \rbr*{ \Pr(\bar f(X,a) =v_\ell \mid A=a) - \tilde q_a(\ell) }} \\
  &= \max_j\envert*{\sum_{\ell=1}^j \rbr*{ \sum_{m=1}^k p_a(v_m) \Pr(g_a(v_m) =v_\ell\mid A=a) - \sum_{m=1}^k \tilde p_a(v_m) \Pr(g_a(v_m) =v_\ell\mid A=a) }} \\
  &= \max_j\envert*{\sum_{m=1}^{k} \rbr*{ p_a(v_m)  -  \tilde p_a(v_m) }  \sum_{\ell=1}^j \Pr(g_a(v_m) =v_\ell\mid A=a)  } \\
  &\leq \max_j\rbr*{\envert*{\sum_{m=1}^{v_{\iota_j-1}} \rbr*{ p_a(v_m)  -  \tilde p_a(v_m) } } +   \envert*{\rbr*{ p_a(v_{\iota_j})  -  \tilde p_a(v_{\iota_j}) }\Pr(g_a(v_{\iota_j}) \leq v_j) } } \\
  &\leq \dKS(p_a, \tilde p_a) +  \max_j \envert*{ p_a(v_{\iota_j})  -  \tilde p_a(v_{\iota_j})  } \\
  &\leq O\rbr*{ \frac{\sqrt{k}}{n w_a\varepsilon}\ln\frac{k|\calA|}{\beta} +  \sqrt{\frac{k}{nw_a}\ln\frac{k|\calA|}\beta}  };
\end{align}
line 3 is because $g_a$ is a transport from $\tilde p_a$ to $\tilde q_a$, line 5 uses \cref{lem:transport.property} and the fact that $\sum_{\ell=1}^{j}\nicefrac{\pi_a(v_m,v_\ell)}{\tilde p_a(v_m)}=\Pr(g_a(v_m)\leq v_j)$, and line 7 is by \cref{thm:private.estimates}.
\end{proof}

\end{document}